\documentclass{article}

\usepackage[preprint,nonatbib]{neurips_2022}

\usepackage[utf8]{inputenc} % allow utf-8 input
\usepackage[T1]{fontenc}    % use 8-bit T1 fonts
\usepackage[backref=page]{hyperref}       % hyperlinks
\usepackage{url}            % simple URL typesetting
\usepackage{booktabs}       % professional-quality tables
\usepackage{amsfonts}       % blackboard math symbols
\usepackage{nicefrac}       % compact symbols for 1/2, etc.
\usepackage{microtype}      % microtypography
\usepackage{xcolor}         % colors

\usepackage{xspace}
\usepackage{graphicx}
\usepackage{multirow}

\usepackage{amssymb}
\usepackage{amsmath}
\usepackage{amsthm}
\usepackage{floatrow}
\usepackage{algorithm}
\usepackage{algpseudocode}
\usepackage{wrapfig}

\usepackage[font=small,labelfont=bf]{caption}

\definecolor{todored}{HTML}{FF0000}
\definecolor{todored}{HTML}{FF0000}

\newcommand{\figpos}[1]{\textbf{\textit{#1}}}
\newcommand{\figtitle}[1]{\textbf{#1}}

\newcommand{\appendixref}[1]{Appendix~\ref{#1}}

\renewcommand{\t}{\tau}
\newcommand{\treal}{\t_{\mathrm{hl}}}
\newcommand{\tdiscovered}{\t_{\mathrm{d}}}
\newcommand{\tnet}[1]{t_{#1}}

\newcommand{\Tset}{T}
\newcommand{\Tsetnamed}[1]{T_{\mathrm{#1}}}
\newcommand{\Treal}{\Tset_{\mathrm{HL}}}
\newcommand{\AS}{\mathrm{AS}}
\newcommand{\ACC}{\mathrm{ACC}}
\newcommand{\E}{\mathbb{E}}
\newcommand{\A}{\mathcal{A}}
\newcommand{\xeloss}{l_{\mathrm{ce}}}

\newcommand{\Xt}{X_{\mathrm{tr}}}
\newcommand{\Xtadv}{X_{\mathrm{tr}}^{\mathrm{adv}}}

\newcommand{\Xte}{X_{\mathrm{te}}}
\newcommand{\Xteadv}{X_{\mathrm{te}}^{\mathrm{adv}}}
\newcommand{\Xspace}{\mathcal{X}}
\newcommand{\Yspace}{\{0, 1\}}
\newcommand{\D}{\mathcal{D}}
\newcommand{\R}{\mathbb{R}}
\newcommand{\similarity}{\mathrm{sim}}
\newcommand{\secref}[1]{Sec.~\ref{#1}}
\newcommand{\tabref}[1]{Tab.~\ref{#1}}
 
\DeclareMathOperator*{\argmax}{\arg\max}

\newcommand{\AStext}{AS\xspace}
\newcommand{\taskdiscovery}{task discovery\xspace}

\definecolor{blue}{HTML}{4878D0}
\definecolor{violet}{HTML}{956CB4}
\definecolor{green}{HTML}{6ACC64}
\definecolor{orange}{HTML}{EE854A}
\definecolor{red}{HTML}{D65F5F}
\newcommand{\green}[1]{\textcolor{green}{#1}}
\newcommand{\blue}[1]{\textcolor{blue}{#1}}
\newcommand{\violet}[1]{\textcolor{violet}{#1}}
\newcommand{\orange}[1]{\textcolor{orange}{#1}}
\newcommand{\red}[1]{\textcolor{red}{#1}}

\renewcommand{\eqref}[1]{Eq.~\ref{#1}}
\newcommand{\figref}[2]{Fig.~\ref{#1}#2}

\newtheorem{prop}{Proposition}

\newcommand{\myparagraph}[1]{\textbf{#1}}
\newcommand{\figcutspace}{\vspace{-0.4cm}}

\title{Task Discovery: Finding the Tasks that\\Neural Networks Generalize on}

\author{%
    \hspace{-4pt}Andrei Atanov\;\; Andrei Filatov\;\; Teresa Yeo\;\; Ajay Sohmshetty\;\; Amir Zamir\vspace{7pt}\\
	Swiss Federal Institute of Technology (EPFL)\\\\
	\url{https://taskdiscovery.epfl.ch}
}

\begin{document}

\maketitle

\begin{abstract}
When developing deep learning models, we usually decide what task we want to solve then search for a model that generalizes well on the task. An intriguing question would be: \emph{what if, instead of fixing the task and searching in the model space, we fix the model and search in the task space?} Can we find tasks that the model generalizes on? How do they look, or do they indicate anything? These are the questions we address in this paper. 

We propose a \textit{task discovery} framework that automatically finds examples of such tasks via optimizing a generalization-based quantity called \textit{agreement score}.
We demonstrate that one set of images can give rise to many tasks on which neural networks generalize well. These tasks are a reflection of the \emph{inductive biases} of the learning framework and the \emph{statistical patterns present in the data}, thus they can make a useful tool for analysing the neural networks and their biases. As an example, we show that the discovered tasks can be used to automatically create \emph{``adversarial train-test splits''} which make a model fail at test time, \emph{without changing the pixels or labels}, but by only selecting how the datapoints should be split between the train and test sets. We end with a discussion on human-interpretability of the discovered tasks.

\end{abstract}

\section{Introduction}
% \begin{wrapfigure}{r}{0.22\textwidth}
% \vspace{-1.4em}
% \includegraphics[width=\linewidth]{figures/arxiv/pull-figure-vertical.pdf} 
% \caption{
% {Task Discovery Framework}
% }
% \label{fig:wrapfig}
% \end{wrapfigure}
Deep learning models are found to generalize well, i.e., exhibit low test error when trained on human-labelled tasks.
This can be seen as a consequence of the models' inductive biases that favor solutions with low test error over those which also have low training loss but higher test error.
In this paper, we aim to find what are examples of other tasks\footnote{In the context of this paper, generally, a ``task'' is a labelling of a dataset, and any label set defines a ``task''.} that are favored by neural networks, i.e., on which they generalize well. We will also discuss some of the consequences of such findings.

We start by defining a quantity called \textit{agreement score} (\AStext) to measure how well a network generalizes on a task.
It quantifies whether two networks trained on the same task with different training stochasticity (e.g., initialization) make similar predictions on new test images.
Intuitively, a high \AStext can be seen as a necessary condition for generalization, as there cannot be generalization if the \AStext is low and networks converge to different solutions.
On the other hand, if the \AStext is high, then there is a stable solution that different networks converge to, and, therefore, generalization is \emph{possible} (see \appendixref{app:as-accuracy-connection}).
We show that the \AStext indeed makes for a useful metric and differentiates between human- and random-labelled tasks (see \figref{fig:pull}{-center}).

Given the \AStext as a prerequisite of generalization, we develop a \textit{\taskdiscovery} framework that optimizes it and finds new tasks on which neural networks generalize well.
Experimentally, we found that the same images can allow for many different tasks on which different network architectures generalize (see an example in \figref{fig:pull}{-right}).

Finally, we discuss how the proposed framework can help us understand deep learning better, for example, by demonstrating its biases and failure modes.
We use the discovered tasks to split a dataset into train test sets in an \textit{adversarial} way instead of random splitting.
After training on the train set, the network fails to make correct predictions on the test set.
The adversarial splits can be seen as using the ``spurious'' correlations that exist in the datasets via discovered tasks.
We conjecture that these tasks provide strong adversarial splits since \taskdiscovery finds tasks that are ``favoured'' by the network the most.
Unlike manually curated benchmarks that reveal similar failing behaviours \cite{sagawa2019distributionally,liang2022metashift,wiles2021fine}, or pixel-level adversarial attacks \cite{szegedy2013intriguing,kurakin2016adversarial,madry2017towards}, the proposed approach finds the adversarial split automatically and \emph{does not need to change any pixels or labels} or collect new difficult images.
\begin{figure}
    {\centering
    \includegraphics[width=\textwidth]{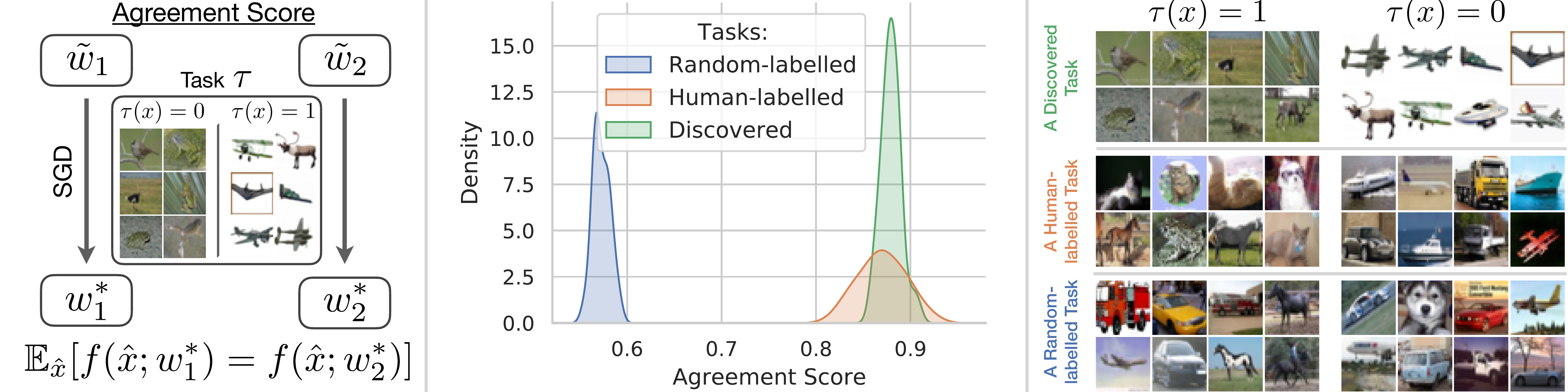}
    \caption{
    \figpos{Left:} The \emph{agreement score} (\AStext) measures whether two networks trained on the task $\tau$ with different optimization stochasticity (e.g., initialization) make the same prediction on a test image $\hat{x}$.
    $f(\hat{x}, w)$ denotes a network with weights $w$ applied to an input $\hat{x}$.
    $\Tilde{w}$ and $w^*$ denote initial and converged weights, respectively.
    % \AStext computation, used as a measure of generalization on a task $\t$.
    \figpos{Center:}
    The agreement score successfully differentiates between \orange{human-labelled} tasks based on CIFAR-10 original labels and \blue{random-labelled} tasks.
    The \green{\emph{\taskdiscovery}} framework finds novel tasks with high agreement scores.
    \figpos{Right:} Examples of a \green{discovered}, \orange{human-labelled} (animals vs. non-animals) and \blue{random-labelled} tasks on CIFAR-10.
    This particular discovered task looks visually distinct and seems to be based on the image background.
    \figcutspace
    }
    \label{fig:pull}}
\end{figure}

\vspace{-5pt}
\section{Related Work}
\vspace{-5pt}
\myparagraph{Deep learning generalization.}
It is not well understood yet why deep neural networks generalize well in practice while being overparameterized and having large complexity 
\cite{hornik1989multilayer, hardt2016identity, nguyen2018optimization, yun2019small, vershynin2020memory}.
A line of work approaches this question by investigating the role of the different components of deep learning \cite{arora2019implicit, neyshabur2014search, arpit2017closer, arora2019implicit, morcos2018insights, zhao2018bias, toneva2018empirical, hacohen2020let} and developing novel complexity measures that could predict neural networks' generalization \cite{neyshabur2015norm,keskar2016large,liang2019fisher,nagarajan2019generalization,smith2017bayesian}.
The proposed \taskdiscovery framework can serve as an experimental tool to shed more light on when and how deep learning models generalize.
Also, compared to existing quantities for studying generalization, namely \cite{zhang2021understanding}, which is based on the \emph{training process only}, our agreement score is directly \emph{based on test data and generalization}.

\myparagraph{Similarity measures between networks.}
Measuring how similar two networks are is challenging and depends on a particular application.
Two common approaches are to measure the similarity between hidden representations \cite{li2015convergent, kornblith2019similarity, laakso2000content, raghu2017svcca, lenc2015understanding, morcos2018insights} and the predictions made on unseen data \cite{somepalli2022can}.
Similar to our work, \cite{hacohen2020let, jiang2021assessing, pliushch2021deep} also use an agreement score measure.
In contrast to these works, which mainly use the \AStext as a metric, we turn it into an optimization objective to find new tasks with the desired property of good generalization.

\myparagraph{Bias-variance decomposition.}
The test error, a standard measure of generalization, can be decomposed into bias and variance, known as the bias-variance decomposition \cite{geman1992neural, bishop2006pattern, valentini2004bias, domingos2000unified}.
The \AStext used in this work can be seen as a measure of the variance term in this decomposition.
Recent works investigate how this term behaves in modern deep learning models \cite{belkin2019reconciling, nakkiran2021deep, yang2020rethinking, neal2018modern}.
In contrast, we characterize its dependence on the task being learned instead of the model's complexity and find tasks for which the \AStext is maximized.

\myparagraph{Meta-optimization.} Meta-optimization methods seek to optimize the parameters of another optimization method.
They are commonly used for hyper-parameter tuning \cite{maclaurin2015gradient, liu2018darts,pedregosa2016hyperparameter, behl2019alpha} and gradient-based meta-learning \cite{finn2017model, weng2018metalearning, nichol2018reptile}. 
We use meta-optimization to find the task parameters that maximize the agreement score, the computation of which involves training two networks on this task.
Meta-optimization methods are memory and computationally expensive, and multiple solutions were developed \cite{lorraine2020optimizing, luketina2016scalable, nichol2018first} to amortize these costs, which can also be applied to improve the efficiency of the proposed \taskdiscovery framework.

\myparagraph{Creating distribution shifts.} Distributions shifts can result from e.g. spurious correlations, undersampling~\cite{wiles2021fine} or adversarial attacks~\cite{szegedy2013intriguing,kurakin2016adversarial,madry2017towards}.
To create such shifts to study the failure modes of networks, one needs to define them manually \cite{sagawa2019distributionally,liang2022metashift}.
We show that it is possible to \textit{automatically} create \textit{many} such shifts that lead to a significant accuracy drop on a given dataset using the discovered tasks (see \secref{sec:adversarial-split} and \figref{fig:adversarial-split}{-left}).
% We show that it is possible to create \textit{many} such shifts that lead to significant accuracy drop (see \secref{sec:adversarial-split} and \figref{fig:adversarial-split}{-left}) \textit{automatically} on a given dataset using the discovered tasks.
In contrast to other automatic methods to find such shifts, the proposed approach allows to creates many of them for a learning algorithm rather than for a single trained model \cite{eyuboglu2022domino}, does not require additional prior knowledge \cite{pmlr-v139-creager21a} or adversarial optimization \cite{lahoti2020fairness}, and does not change pixel values \cite{kurakin2016adversarial} or labels.

\paragraph{Data-centric analyses of learning.}
Multiple works study how training data influences the final model.
Common approaches are to measure the effect of removing, adding or mislabelling individual data points or a group of them \cite{koh2017understanding, koh2019accuracy, ilyas2022datamodels, pmlr-v89-jia19a}.
Instead of choosing which objects to train on, in task discovery, we choose how to label a fixed set of objects (i.e., tasks), s.t. a network generalizes well.

\paragraph{Transfer/Meta-/Multi-task learning.}
This line of work aims to develop algorithms that can solve multiple tasks with little additional supervision or computational cost.
They range from building novel methods \cite{jaegle2021perceiver, chowdhery2022palm, finn2017model, vinyals2016matching, garnelo2018conditional} to better understanding the space of tasks and their interplay \cite{taskonomy2018, bao2019information, raghu2019transfusion, nguyen2020leep, achille2019task2vec}.
Most of these works, however, focus on a small set of human-defined tasks.
With task discovery, we aim to build a more extensive set of tasks that can further facilitate these fields' development. 

% \vspace{-5pt}
\section{Agreement Score: Measuring Consistency of Labeling Unseen Data}
% \vspace{-5pt}
% In this section, we introduce the basis for the \taskdiscovery framework: a task definition and a measure of the task's generalizability, the agreement score.
In this section, we introduce the basis for the \taskdiscovery framework: the definitions of a task and the agreement score, which is a measure of the task's generalizability.
We then provide an empirical analysis of this score and demonstrate that it differentiates human- and random-labelled tasks.

\textbf{Notation and definitions.}
Given a set of $N$ images $X = \{x_i\}_{i=1}^N, \, x \in \Xspace$, we define a \textit{task} as a binary labelling of this set $\t: X \to \Yspace$ (a multi-class extension is straightforward) and denote the corresponding labelled dataset as $\D(X, \t) = \{(x, \t(x)) | x \in X\}$.
% and the set of all possible tasks (i.e., label sets) on $X$ as $\T_{X}$.
We consider a learning algorithm $\A$ to be a neural network $f(\cdot; w):\Xspace \to [0, 1]$ with weights $w$ trained by SGD with cross-entropy loss.
Due to the inherent stochasticity, such as random initialization and mini-batching, this learning algorithm induces a distribution over the weights given a dataset $w \sim \A(\D(X, \t))$.

% \vspace{-3pt}
\subsection{Agreement Score as a Measure of Generalization}
\label{sec:as}
% \vspace{-3pt}
A standard approach to measure the generalization of a learning algorithm $\A$ trained on $\D(\Xt, \t)$ is to measure the test error on $\D(\Xte, \t)$.
The test error can be decomposed into bias and variance terms \cite{bishop2006pattern, valentini2004bias, domingos2000unified},
where, in our case, the stochasticity is due to $\A$ while the \textit{train-test split is fixed}.
% Many works study how this decomposition depends on the choice of $\A$ \cite{neal2018modern, yang2020rethinking}, whereas we examine its dependence on the task $\t$ for a fixed $\A$.
We now examine how this decomposition depends on the task $\t$.
The bias term measures how much the average prediction deviates from $\t$ and mostly depends on what are the test labels on $\Xte$.
The variance term captures how predictions of different models agree with each other and does not depend on the task's test labels but only on training labels through $\D(\Xt, \t)$.
We suggest measuring the generalizability of a task $\t$ with an \textit{agreement score}, as defined below.

For a given train-test split $\Xt, \Xte$ and a task $\t$, we define the agreement score (\AStext) as:
\begin{equation}
    \label{eq:as}
    \AS(\t; \Xt, \Xte) = \E_{w_1, w_2 \sim \A(\D(\Xt, \t))} \E_{x\sim \Xte} [f(x; w_1) = f(x; w_2)],
\end{equation}
where the first expectation is over different models trained on $\D(\Xt, \t)$ and the inner expectation is averaging over the test set.
Practically, this corresponds to training two networks from different initializations on the training dataset labelled by $\t$ and measuring the agreement between these two models' predictions on the test set (see \figref{fig:pull}{-left} and the inner-loop in \figref{fig:method}{-left}).

The \AStext depends only on the task's training labels, thus, test labels are not required.
In \appendixref{app:as-accuracy-connection}, we provide a more in-depth discussion on the connection between the \AStext of a task and the test accuracy after training on it.
We highlight that in general high \AStext does not imply high test accuracy as we also demonstrate in \secref{fig:adversarial-split}.
% However, it is tightly connected to the test error, as having a high-\AStext task $\t$ that labels training data, one can construct a task with a high test accuracy.
% We note, however, that when the task is given and fixed, the high-\AStext provides only a necessary but not a sufficient condition for a high test accuracy as test labels can take any values in general (e.g., be adversarial as in \secref{fig:adversarial-split}), thus, it cannot be used to predict the test accuracy, in this case \cite{jiang2021assessing, kirsch2022note}.
% Refer to \appendixref{app:as-accuracy-connection} for a more in-depth discussion.

\vspace{-5pt}
\subsection{Agreement Score Behaviour for Random- and Human-Labelled Tasks}
\vspace{-5pt}
\label{sec:as-real-random-study}

\begin{figure}
    \centering
    \includegraphics[width=\textwidth]{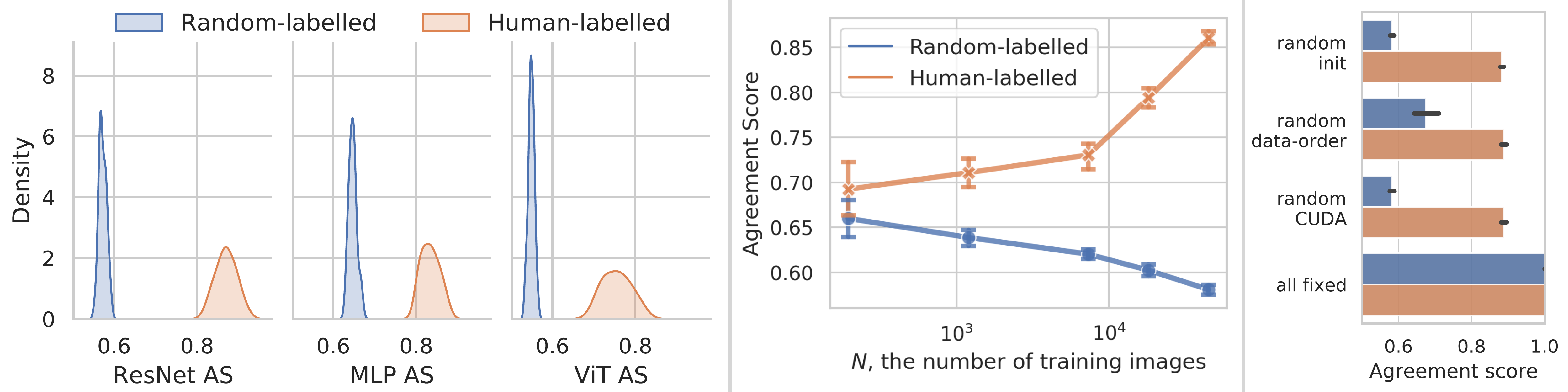}
    \caption{
    \figtitle{Agreement score for \orange{human}- and \blue{random}-labelled tasks.}
    \figpos{Left:} \AStext measured on three architectures: ResNet-18, MLP and ViT.
    \figpos{Center:} \AStext measured on ResNet-18 for different numbers of training images $N$. The standard deviation is over four random tasks and three data splits.
    \figpos{Right:} Ablating the sources of stochasticity present in $\A$.
    Each row shows when one of the 1) initialization, 2) data-order or 3) CUDA is stochastic and the other two sources are fixed, and the bottom is when all the sources are fixed (see \secref{sec:as-real-random-study}).
    {The differentiation between \orange{human}- and \blue{random}-labelled tasks stably persists across different architectures, data sizes and sources of randomness.}
    \figcutspace
    }
    \label{fig:as-study}
\end{figure}

In this section, we demonstrate that the \AStext exhibits the desired behaviour of differentiating human-labelled from random-labelled tasks. 
We take the set of images $X$ from the CIFAR-10 dataset \cite{krizhevsky2009learning} and split the original training set into 45K images for $\Xt$ and 5K for $\Xte$.
We split 10 original classes differently into two sets of 5 classes to construct 5-vs-5 binary classification tasks $\treal$.
Out of  all $\binom{10}{5}$ tasks, we randomly sub-sample 20 to form the set of \textit{human-labelled tasks} $\Tsetnamed{HL}$.
We construct the set of 20 \textit{random-labelled tasks} $\Tsetnamed{RL}$ by generating binary labels for all images randomly and fixing them throughout the training, similar to \cite{zhang2021understanding}.
We use ResNet-18 \cite{he2016deep} architecture and Adam \cite{kingma2014adam} optimizer as the learning algorithm $\A$, unless otherwise specified.
We measure the \AStext by training two networks for 100 epochs, which is enough to achieve zero training error for all considered tasks.

\myparagraph{\AStext differentiates human- from random-labelled tasks.}
\figref{fig:pull}{-center} shows that human-labelled tasks have a higher \AStext than random-labelled ones.
This coincides with our intuition that one should not expect generalization on a random task, for which \AStext is close to the chance level of 0.5.
Note, that the empirical distribution of the \AStext for random-labelled tasks (\figref{fig:as-study}{-left}) is an unbiased estimation of the \AStext distribution over all possible tasks, as $\Tsetnamed{RL}$ are uniform samples from the set of all tasks (i.e., labelings) defined on $\Xt$.
This suggests that high-\AStext tasks do not make for a large fraction of all tasks.
% Note, however, that random tasks $\Tsetnamed{RL}$ are just samples from $\T_{\Xt}$ (which also includes $\Tsetnamed{HL}$) and the corresponding empirical distribution of \AStext (\figref{fig:as-study}{-left}) is an unbiased estimation of the \AStext distribution over $\T_{\Xt}$, suggesting that high-\AStext tasks do not make for a large fraction of all tasks.
A high \AStext of human-labelled tasks is also consistent with the understanding that a network is able to generalize when trained on the original CIFAR-10 labels.

\myparagraph{How does the AS differ across architectures?}
In addition to ResNet-18, we measure the \AStext using MLP \cite{somepalli2022can} and ViT architectures \cite{dosovitskiy2020image} (see \appendixref{app:arch}).
\figref{fig:as-study}{-left} shows that the \AStext for all the architectures distinguishes tasks from $\Tsetnamed{RL}$ and $\Tsetnamed{HL}$ well. 
MLP and ViT have lower \AStext than ResNet-18, aligned with our understanding that convolutional networks should exhibit better generalization on this small dataset due to their architectural inductive biases.
Similar to the architectural analysis, we provide the analysis on the data-dependence of the \AStext in \secref{sec:as-test-data-domain}.

\myparagraph{How does the AS depend on the training size?}
When too little data is available for training, one could expect the stochasticity in $\A$ to play a bigger role than data, i.e., the agreement score decreases with less data.
This intuition is consistent with empirical evidence for human-labelled tasks, as seen in \figref{fig:as-study}{-center}.
However, for random-labelled tasks, the \AStext increases with less data (but they still remain distinguishable from human-labelled tasks).
One possible reason is that when the training dataset is very small, and the labels are random, the two networks do not deviate much from their initializations. This results in similar networks and consequently a high \AStext, but basically irrelevant to the data and uninformative.

\myparagraph{Any stochasticity is enough to provide differentiation.} 
We ablate the three sources of stochasticity in $\A$: 1) initialization, 2) data-order and 3) CUDA stochasticity \cite{jooybar2013deterministic}.
% \footnote{\todo{link to CUDA website}}.
The system is fully deterministic with all three sources fixed and $\AS$=$1$.
Interestingly, when \textit{any} of these variables change between two runs, the \AStext drops, creating the same separation between human- and random-labelled tasks.

These empirical observations show that the \AStext well differentiates between human- and random-labelled tasks across multiple architectures, dataset sizes and the sources of stochasticity.

\section{Task Discovery via Meta-Optimization of the Agreement Score}
\label{sec:td}
As we saw in the previous section, \AStext provides a useful measure of how well a network can generalize on a given task.
A natural question then arises: \textit{are there high-AS tasks other than human-labelled ones, and what are they?}
In this section, we establish a \textit{\taskdiscovery} framework to automatically search for these tasks and study this question computationally.

\subsection{Task Space Parametrization and Agreement Score Meta-Optimization}
\label{sec:as-meta-opt}
Our goal is to find a task $\t$ that maximizes $\AS(\t)$.
It is a high-dimensional discrete optimization problem which is computationally hard to solve.
In order to make it differentiable and utilize a more efficient first-order optimization, we, first, substitute the 0-1 loss in \eqref{eq:as} with the cross-entropy loss $\xeloss$.
Then, we parameterize the space of tasks with a \textit{task network} $\tnet{\theta}: \Xspace \to [0, 1]$ and treat the \AStext as a function of its parameters $\theta$.
This basically means that we view the labelled training dataset $\D(\Xt, \t)$ and a network $\tnet{\theta}$ with the same labels on $\Xt$ as being equivalent, as one can always train a network to fit the training dataset.

Given that all the components are now differentiable, we can calculate the gradient of the \AStext w.r.t. task parameters $\nabla_\theta AS(\tnet{\theta})$ by unrolling the inner-loop optimization and use it for gradient-based optimization over the task space.
This results in a meta-optimization problem where the inner-loop optimization is over parameters $w_1, w_2$ and the outer-loop optimization is over task parameters $\theta$ (see \figref{fig:method}{-left}).

Evaluating the meta-gradient w.r.t. $\theta$ has high memory and computational costs, as we need to train two networks from scratch in the inner-loop.
To make it feasible, we limit the number of inner-loop optimization steps to 50, which we found to be enough to separate the \AStext between random and human-labelled tasks and provide a useful learning signal (see \appendixref{app:proxy-as}).
We additionally use gradient checkpointing \cite{chen2016training} after every inner-loop step to avoid linear memory consumption in the number of steps.
This allows us to run the discovery process for the ResNet-18 model on the CIFAR-10 dataset using a single 40GB A100.
See \secref{app:meta-opt} for more details.

\subsection{Discovering \emph{a Set} of Dissimilar Tasks with High Agreement Scores}
\label{sec:td-general-formulation}

The described \AStext meta-optimization results in only a single high-\AStext task, whereas there are potentially many tasks with a high \AStext.
Therefore, we would like to discover a set of tasks $\Tset = \{\tnet{\theta_1}, \dots, \tnet{\theta_K}\}$.
A naive approach to finding such a set would be to run the meta-optimization from $K$ different intializations of task-parameters $\theta$.
However, this results in a set of similar (or even the same) tasks, as we observed in the preliminary experiments.
Therefore, we aim to discover \textit{dissimilar} tasks to represent the set of all tasks with high \AStext better.
We measure the similarity between two tasks on a set $X$ as follows:
\begin{equation}
    \label{eq:sim}
    \similarity_X(\t_1, \t_2) = \max\left\{ \E_{x\sim X} [\t_1(x) = \t_2(x)], \E_{x\sim X} [\t_1(x) = 1 - \t_2(x)]\right\},
\end{equation}
where the maximum accounts for the labels' flipping.
Since this metric is not differentiable, we, instead, use a differentiable loss  $L_{\mathrm{sim}}$ to minimize the similarity (defined later in \secref{sec:td-encoder-formulation}, \eqref{eq:uniformity}).
Finally, we formulate the \taskdiscovery framework as the following optimization problem over $\Tset$:
\begin{equation}
    \label{eq:task-discovery-obj}
    \argmax_{\Tset = \{\tnet{\theta_1}, \dots, \tnet{\theta_K}\}}  \E_{\tnet{\theta} \sim \Tset} \AS(\tnet{\theta}) - \lambda \cdot L_{\mathrm{sim}}(\Tset).
\end{equation}
We show the influence of $\lambda$ on task discovery in \appendixref{app:lambda-hyperparam-as-similarity-curve}.
Note that this naturally avoids discovering trivial solutions that are highly imbalanced (e.g., labelling all objects with class 0) due to the similarity loss, as these tasks are similar to each other, and a set $\Tset$ that includes them will be penalized.

In practice, we could solve this optimization sequentially -- i.e., first find $\tnet{\theta_1}$, freeze it and add it to $\Tset$, then optimize for $\tnet{\theta_2}$, and so on. However, we found this to be slow. 
In the next section \ref{sec:td-encoder-formulation}, we provide a solution that is more efficient, i.e., can find more tasks with less compute.

\myparagraph{\emph{Regulated} task discovery.}
The task discovery formulation above only concerns with finding high-\AStext tasks -- which is the minimum requirement for having generalizable tasks.
One can introduce additional constraints to \textit{regulate/steer} the discovery process, e.g., by adding a regularizer in Eq.~\ref{eq:task-discovery-obj} or via the task network's architectural biases.
This approach can allow for a guided task discovery to favor the discovery of some tasks over the others.
We provide an example of a regulated discovery in \secref{sec:td-study-qualitative} by using self-supervised pre-trained embeddings as the input to the task network.

\subsection{Modelling The Set of Tasks with an Embedding Space}
\label{sec:td-encoder-formulation}
\begin{figure}
    \centering
    \includegraphics[width=\textwidth]{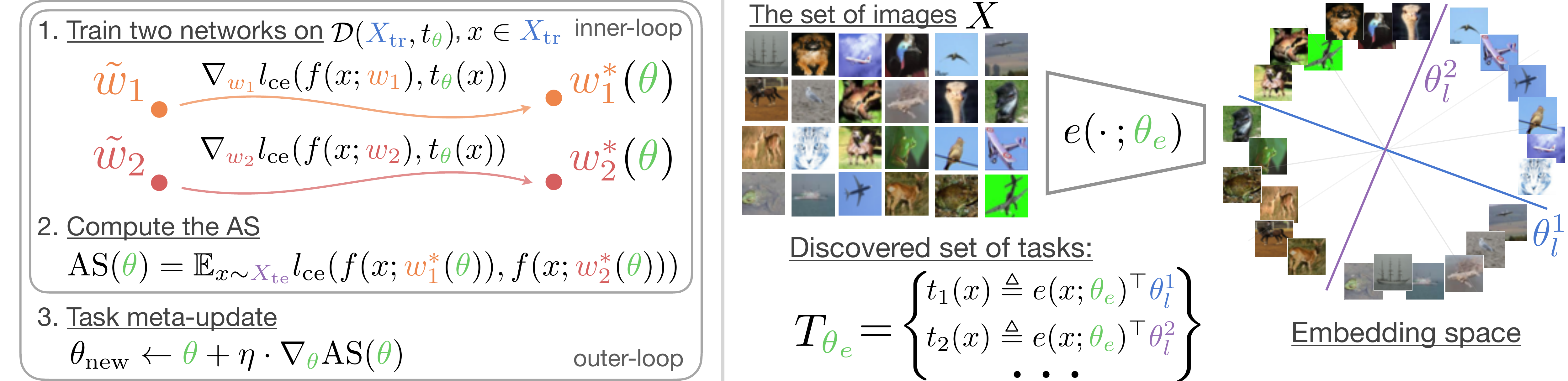}
    \caption{
    % \todo{This can be smaller to save space}
    \figpos{Left:}
    Agreement score meta-optimization.
    \textit{Inner-loop:} given the task $t_{\green{\theta}} \in \Tset$, two networks $\orange{w_1}, \red{w_2}$ optimize the cross-entropy loss $\xeloss$ on the training set labelled by the task $t_{\green{\theta}}$ (\secref{sec:as-meta-opt}).
    After training, the agreement between two networks is calculated on the test set.
    \textit{Outer-loop:} the task parameters $\green{\theta}$ are updated with the meta-gradient of the \AStext.
    \figpos{Right:} the task-space paramertrization with the shared encoder (\secref{sec:td-encoder-formulation}).
    The encoder $e(\cdot\,;\green{\theta_e})$ (after projection) distributes images uniformly on the sphere in the embedding space.
    Different tasks are then formed by applying a linear classifier with weights, e.g., $\violet{\theta_l^1}$ and $\blue{\theta_l^2}$, passing through the origin.
    The corresponding set of tasks $\Tset_{\green{\theta_e}}$ consists of \emph{all} such linear classifiers in this space.
    This results in a more efficient \taskdiscovery framework than modelling each task with a separate network.
    \figcutspace
    }
    \label{fig:method}
\end{figure}

Modelling every task by a separate network increases memory and computational costs.
In order to amortize these costs, we adopt an approach popular in multi-task learning and model task-networks with a shared encoder $e(\cdot; \theta_{e}): \Xspace \to \R^d$ and a task-specific linear head $\theta_{l} \in \R^d$, so that $t_\theta(x) = e(x; \theta_{e})^\top\theta_{l}$.
See \figref{fig:method}{-right} for visualization.
Then, instead of learning a fixed set of different task-specific heads, we aim to learn an embedding space where \textit{any} linear hyperplane gives rise to a high-\AStext task.
Thus, an encoder with parameters $\theta_e$ defines the following set of tasks:\begin{equation}
    \label{eq:embedding-tasks}
    \Tset_{\theta_e} =\{\tnet{\theta}\,|\, \theta = (\theta_e, \theta_{l}), \, \theta_{l}\in\R^d\}.
\end{equation}
This set is not infinite as it might seem at first since many linear layers will correspond to the same shattering of $X$.
The size of $\Tset_{\theta_e}$ as measured by the number of unique tasks on $X$ is only limited by the dimensionality of the embedding space $d$ and the encoder's expressivity.
Potentially, it can be as big as the set of all shatterings when $d=|X|-1$ and the encoder $e$ is flexible enough \cite{vapnik2015uniform}.

We adopt a uniformity loss over the embedding space \cite{wang2020understanding} to induce dissimilarity between tasks:
\begin{equation}
    \label{eq:uniformity}
    L_{\mathrm{sim}}(\Tset_{\theta_e}) = L_{\mathrm{unif}}(\theta_e) = \log \E_{x_1, x_2} \exp\left\{\alpha \cdot \frac{e(x_1; \theta_e)^\top e(x_2; \theta_e)}{\|e(x_1; \theta_e)\| \cdot \|e(x_2; \theta_e)\|} \right\},
\end{equation}
where the expectation is taken w.r.t. pairs of images randomly sampled from the training set.
It favours the embeddings to be uniformly distributed on a sphere, and, if the encoder satisfies this property, any two orthogonal hyper-planes $\theta_l^1 \perp \theta_l^2$ (passing through origin) give rise to two tasks with the low similarity of 0.5 as measured by \eqref{eq:sim} (see \figref{fig:method}{-right}).

In order to optimize the \AStext averaged over $\Tset_{\theta_e}$ in \eqref{eq:task-discovery-obj}, we can use a Monte-Carlo gradient estimator and sample one hyper-plane $\theta_l$ at each step, e.g., w.r.t. an isotropic Gaussian distribution, which, on average, results in dissimilar tasks given that the uniformity is high.
As a result of running the task discovery framework, we find the encoder parameters $\theta_e^*$ that optimize the objective \eqref{eq:task-discovery-obj} and gives rise to the corresponding set of tasks $\Tset_{\theta_e^*}$ (see \eqref{eq:embedding-tasks} and \figref{fig:method}{-right}).

Note that the framework outlined above can be straightforwardly extended to the case of discovering multi-way classification tasks as we show in \appendixref{app:td-k-way-cifar}, where instead of sampling a hyperplane that divides the embedding space into two parts, we sample $K$ hyperplanes that divide the embedding space into $K$ regions and give rise to $K$ classes.
% We provide an example of discovering $K$-way classification tasks in \appendixref{app:td-k-way-cifar}.

\myparagraph{Towards \textit{the space} of high-\AStext tasks.}
Instead of creating a set of tasks, one could seek to define \textit{a space} of high-\AStext tasks.
That is to define \emph{a basis set} of tasks and a binary operation on a set of tasks that constructs a new task and preserves the \AStext.
The proposed formulation with a shared embedding space can be seen as a step toward this direction.
Indeed, in this case, the structure over $\Tset_{\theta_e^*}$ is imposed implicitly by the space of liner hyperplanes $\theta_l\in\R^d$, each of which gives rise to a high-\AStext task.

\subsection{Discovering High-\AStext Tasks on CIFAR-10 Dataset}
\label{sec:td-cifar}
In this section, we demonstrate that the proposed framework successfully finds dissimilar tasks with high AS.
We consider the same setting as in \secref{sec:as-real-random-study}.
We use the encoder-based discovery described in \secref{sec:td-encoder-formulation} and use the ResNet-18 architecture for $e(\cdot; \theta_e)$ with a linear layer mapping to $\R^d$ (with $d=32$) instead of the last classification layer.
We use Adam as the meta-optimizer and SGD optimizer for the inner-loop optimization.
Please refer to \appendixref{app:meta-opt} for more details.

We optimize \eqref{eq:task-discovery-obj} to find $\theta_e^*$ and sample 32 tasks from $\Tset_{\theta_e^*}$ by taking $d$ orthogonal hyper-planes $\theta_l$.
We refer to this set of 32 tasks as $\Tsetnamed{ResNet}$.
The average similarity (\eqref{eq:sim}) between all pairs of tasks from this set is $0.51$, close to the smallest possible value of $0.5$.
For each discovered task, we evaluate its \AStext in the same way as in \secref{sec:as-real-random-study} (according to \eqref{eq:as}).
\figref{fig:pull}{-center} demonstrates that the proposed \taskdiscovery framework successfully finds tasks with high \AStext.
See more visualizations and analyses on these discovered tasks in \secref{sec:td-study} and \appendixref{app:vizualization}.

\myparagraph{Random networks give rise to high-AS tasks.}
Interestingly, in our experiments, we found that if we initialize a task network randomly (standard uniform initialization from PyTorch \cite{NEURIPS2019_9015}), the corresponding task (after applying softmax) will have a high \AStext, on par with human-labelled and discovered ones ($\approx 0.85$).
Different random initializations, however, give rise to very similar tasks (e.g., 32 randomly sampled networks have an average similarity of $0.68$ compared to $0.51$ for the discovered tasks. See \appendixref{app:random-net} for a more detailed comparison).
Therefore, a naive approach of sampling random networks does not result in an efficient task discovery framework, as one needs to draw prohibitively many of them to get sufficiently dissimilar tasks.
Note, that this result is specific to the initialization scheme used and not \textit{any} instantiation of the network's weights necessarily results in a high-\AStext task.

\vspace{-5pt}
\section{Empirical Study of the Discovered Tasks}
\vspace{-5pt}
\begin{figure}
    \centering
    \includegraphics[width=\textwidth]{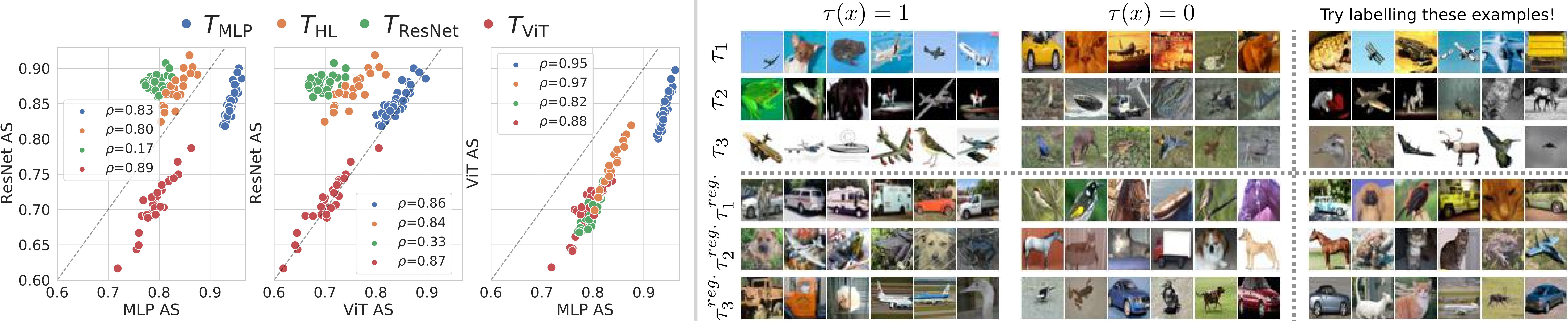}
    \caption{
    % \todo{UPDATE ME}
    % \todo{This figure will be updated for task images to take more space. Layout is final.}
    \figpos{Left:}
    The \AStext for tasks from $\textcolor{orange}{\Tsetnamed{HL}}, \textcolor{green}{\Tsetnamed{{ResNet}}}, \textcolor{blue}{\Tsetnamed{MLP}}, \textcolor{red}{\Tsetnamed{{ViT}}}$ measured on all three architectures ($x$ and $y$ axis).
    \figpos{Right:} Sample images for each class of our discovered tasks are shown in the first two columns. We show some unlabelled examples in the third column for the reader to guess the label. The answers are in the \appendixref{app:vizualization}. These images have been sampled to be the most discriminative, as measured by the network predicted probability. \textit{\textbf{Right}} (top): Some of the discovered tasks from the unregulated version of task discovery seem to be based on color patterns, e.g. $\tau_1(x)=1$ are images with similar blue color, which make sense as a learnable task and are expected (see Sec.~\ref{sec:td-study-qualitative}). \textit{\textbf{Right}} (bottom): The same, but for a regulated version where the encoder was pretrained with SimCLR. The tasks seems to correspond more to semantic tasks e.g. $\tau_1^{reg.}(x)=1$ are images of vehicles with different backgrounds. As the pretraining encourages the encoder to be invariant to color, it seems to be biased towards semantic information. Thus, the framework is able to pick up on the inductive biases from SimCLR pretraining, via the discovered tasks.
    % \textbf{Left:} examples of images from the adversarial split of the CelebA dataset and results of learning on this split vs random split.
    \figcutspace
    }
    \label{fig:td-arch}
\end{figure}
\label{sec:td-study}
In this section, we first perform a qualitative analysis of the discovered tasks in \secref{sec:td-study-qualitative}.
Second, we analyze how discovered tasks differ for three different architectures and test data domains in \secref{sec:td-study-arch} and \secref{sec:as-test-data-domain}, respectively.
Finally in \secref{sec:human-interpretability}, we discuss the human-interpretability of the discovered tasks and whether human-labelled tasks should be expected to be discovered.
% Finally, in \secref{sec:adversarial-split}, we present how the discovered tasks can be used to construct an \textit{adversarial attack} for a network by simply curating a train-test split.

% _id = '1v42d6nl'
% (0,3,2)

% 0 [1, 0, 0, 1, 0, 1]
% 1 [1, 0, 1, 1, 0, 0]
% 2 [0, 1, 1, 1, 0, 0]
% 3 [0, 1, 0, 0, 1, 1]

% _id = 'af1d81fe-91e8-4363-af0b-1b683d42d8a6'
% (0,2,6)

% 0 [0, 1, 0, 1, 1, 0]
% 1 [1, 0, 0, 0, 1, 1]
% 2 [1, 1, 1, 0, 0, 0]
% 3 [1, 0, 0, 1, 1, 0]
% 4 [0, 1, 0, 1, 0, 1]
% 5 [1, 1, 1, 0, 0, 0]
% 6 [0, 0, 1, 1, 1, 0]
% 7 [0, 1, 1, 0, 1, 0]

%%% 
% 0 [0, 1, 0, 1, 1, 0]
% 2 [1, 1, 1, 0, 0, 0]
% 6 [0, 0, 1, 1, 1, 0]
% 0 [1, 0, 0, 1, 0, 1]
% 3 [0, 1, 0, 0, 1, 1]
% 2 [0, 1, 1, 1, 0, 0]

\subsection{Qualitative Analyses of the Discovered Tasks}
\label{sec:td-study-qualitative}
Here we attempt to analyze the tasks discovered by the proposed discovery framework. \figref{fig:td-arch}{-top-right} shows examples of images for three sample discovered tasks found in \secref{sec:td-cifar}. 
For ease of visualization, we selected the most discriminative images from each class as measured by the network's predicted probability. In the interest of space, the visuals for all tasks are in the \appendixref{app:vizualization}, alongside some post hoc interpretability analysis.

Some of these tasks seem to be based on color, e.g. the class 1 of $\tau_1$ includes images with blue (sky or water) color, and the class 0 includes images with orange color. Other tasks seem to pick up other cues. These are basically reflections of the \emph{statistical patterns present in the data} and the \emph{inductive biases of the learning architecture}.

\myparagraph{Regulated \taskdiscovery} described in \secref{sec:td-general-formulation} allows us to incorporate additional information to favor the discovery of specific tasks, e.g., ones based on more high-level concepts.
As an example, we use self-supervised contrastive pre-training that learns embeddings invariant to the employed set of augmentations \cite{hadsell2006dimensionality, he2020momentum, caron2020unsupervised, misra2020self}.
Specifically, we use embeddings of the ResNet-50 \cite{he2016deep} trained with SimCLR \cite{chen2020simple} as the input to the encoder $e$ instead of raw images, which in this case is a 2-layer fully-connected network.
Note that the AS is still computed using the raw pixels.

\figref{fig:td-arch}{-bottom-right} shows the resulting tasks. Since the encoder is now more invariant to color information due to the color jittering augmentation employed during pre-training \cite{chen2020simple}, the discovered tasks seem to be more aligned with CIFAR-10 classes; e.g. samples from $\tau_1$'s class 1 show vehicles against different backgrounds.
Note that task discovery regulated by contrastive pre-training only provides a tool to discover tasks invariant to the set of employed augmentations.
The choice of augmentations, however, depends on the set of tasks one wants to discover.
For example, one should not employ a rotation augmentation if they need to discover view-dependent tasks \cite{xiao2020should}.

\subsection{Dependency of the Agreement Score and Discovered Tasks on the Architecture}
\label{sec:td-study-arch}
In this section, we study how the \AStext of a task depends on the neural network architecture used for measuring the \AStext. We include human-labelled tasks as well as a set of tasks discovered using different architectures in this study.
We consider the same architectures as in \secref{sec:as-real-random-study}: ResNet-18, MLP, and ViT.
We change both $f$ and $e$ to be one of MLP or ViT and run the same discovery process as in \secref{sec:td-cifar}.
As a result, we obtain three sets: $\green{\Tsetnamed{ResNet}}, \blue{\Tsetnamed{MLP}}, \red{\Tsetnamed{ViT}}$, each with 32 tasks.
For each task, we evaluate its \AStext on all three architectures.
\figref{fig:td-arch}{-left} shows that high-AS tasks for one architecture do not necessary have similar AS when measured on another architecture (e.g., $\textcolor{green}{\Tsetnamed{ResNet}}$ on ViT).
For MLP and ViT architectures, we find that the \AStext groups correlate well for all groups, which is aligned with the understanding that these architectures are more similar to each other than to ResNet-18.

More importantly, we note that comparing architectures on any single group of tasks is not enough.
For example, if comparing the \AStext for ResNet-18 and MLP only on human-labelled tasks $\orange{\Treal}$, one might conclude that they correlate well ($\rho=0.8$), suggesting they generalize on similar tasks.
However, when the set $\green{\Tsetnamed{ResNet}}$ is taken into account, the conclusion changes ($\rho=0.17$).
Thus, it is important to analyse the different architectures on a broader set of tasks not to bias our understanding, and the proposed \taskdiscovery framework allows for more complete analyses.

\subsection{Dependency of the Agreement Score on the Test Data Domain}
\label{sec:as-test-data-domain}
\begin{figure}
    \centering
    \includegraphics[width=\textwidth]{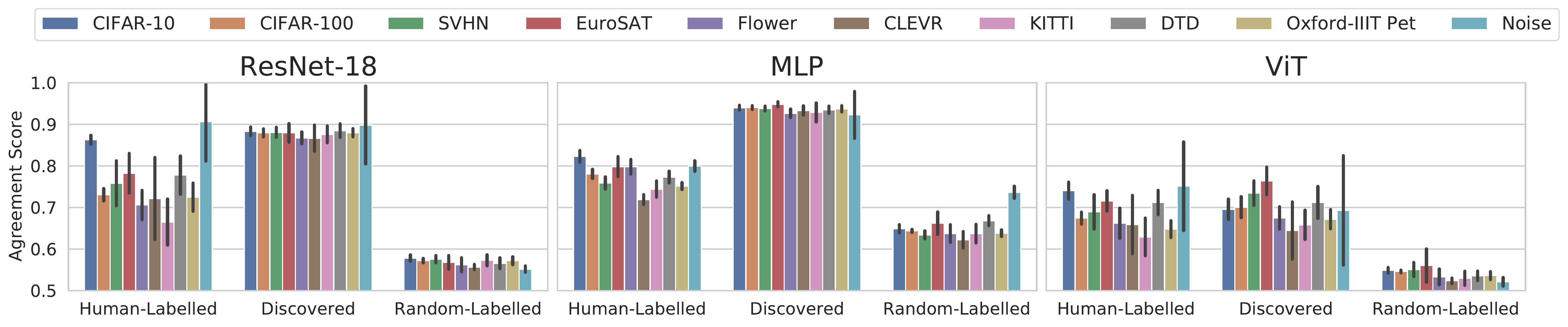}
    \caption{
    \figtitle{The dependency of the agreement score on the test data domain.}
    We fix training images $\Xt$ to be CIFAR-10 and vary $\Xte$ to measure the \AStext.
    For each architecture, we take five tasks from each set of human-labelled, random-labelled and discovered (for the same architecture) tasks and measure their \AStext on different test data domains.
    The standard deviation is over five tasks and three random seeds for each pair of a task type and test domain.
    We see that the differentiation between human-labelled and discovered tasks and random-labelled tasks persists across multiple domains.
    We also find that the \AStext of discovered tasks is more stable and stays high across all the domains, whereas the \AStext of human-labelled tasks is more volatile.
    \figcutspace
    }
    \label{fig:as-ood-data}
\end{figure}

In this section, we study how the \AStext of different tasks depends on the test data used to measure the agreement between the two networks.
We fix $\Xt$ to be images from the CIFAR-10 dataset and take test images $\Xte$ from the following datasets:
CIFAR-100 \cite{krizhevsky2009learning}, CLEVR\cite{johnson2017clevr}, Describable Textures Dataset \cite{cimpoi2014describing}, EuroSAT \cite{helber2019eurosat}, KITTI \cite{geiger2013vision}, 102 Category Flower Dataset \cite{nilsback2008automated}, Oxford-IIIT Pet Dataset \cite{parkhi2012cats}, SVHN \cite{lecun2004learning}.
We also include the AS measured on noise images sampled from a standard normal distribution.
Before measuring the AS, we standardise images channel-wise for each dataset to have zero mean and unit variance similar to training images.

The results for all test datasets and architectures are shown in \figref{fig:as-ood-data}.
First, we can see that the differentiation between human-labelled and discovered tasks and random-labelled tasks successfully persists across all the datasets.
For ResNet-18 and MLP, we can see that the \AStext of discovered tasks stays almost the same for all considered domains.
This might suggest that the patterns the discovered tasks are based on are less domain-dependent -- thus provide reliable cues across different image domains.
On the other hand, human-labelled tasks are based on high-level semantic patterns corresponding to original CIFAR-10 classes, which change drastically across different domains (e.g., there are no animals or vehicles in CLEVR images), and, as a consequence, their \AStext is more domain-dependent and volatile.

% \vspace{-5pt}
\subsection{On Human-Interpretability and Human-Labelled Tasks}
% \vspace{-3pt}
\begin{figure}
    \centering
    \includegraphics[width=0.7\textwidth]{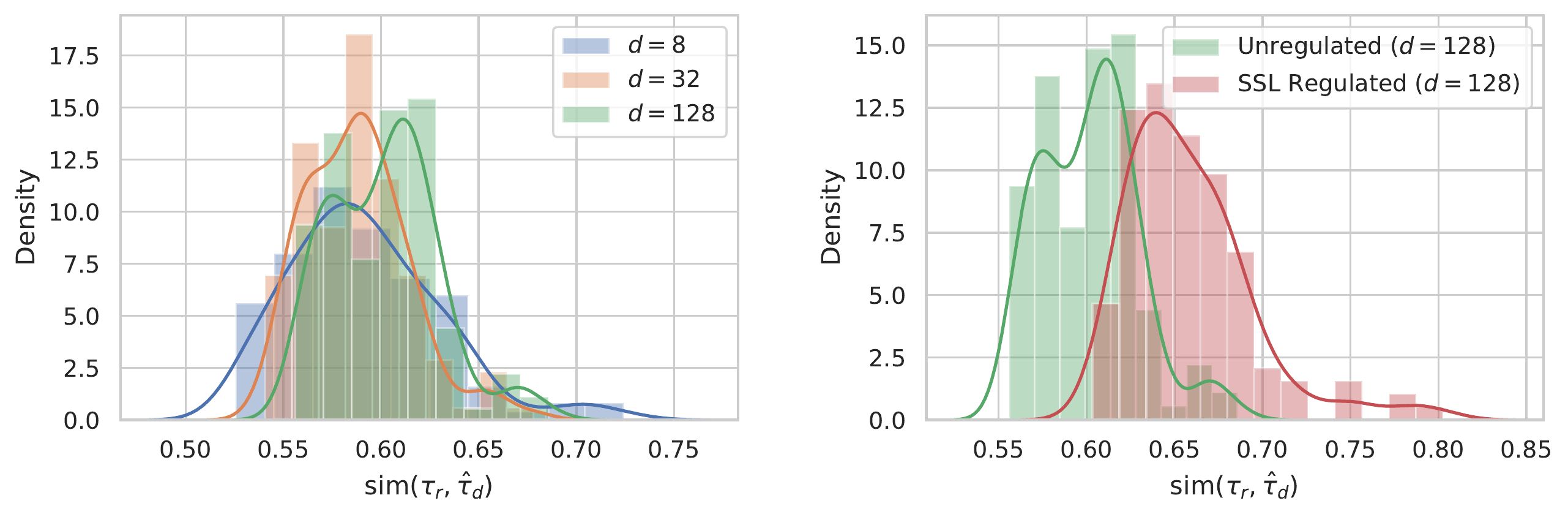}
    \vspace{-0.3cm}
    \caption{
    \figtitle{Are human-labelled tasks present in the set of discovered tasks?}
    For each available human-labelled task $\t  \in \Treal$, we find the most similar discovered task $\hat{\t}_d \in \Tset_{\theta_e^*}$ and plot the corresponding similarities $\similarity(\treal, \hat{\t}_d)$ as a distribution.
    \figpos{Left:} Unregulated task discovery with varying embedding space dimensionalities $d$.
    The similarity stays relatively low for most human-labelled tasks suggesting they are not ``fully'' included in the set of discovered tasks at the scale with which we experimented.
    \figpos{Right:} Unregulated and SSL regulated versions of \taskdiscovery with $d=128$.
    As expected, the tasks discovered by the regulated version are more similar to human-labelled tasks due to the additional inductive biases.
    \figcutspace
    }
    \label{fig:app-real-recall}
\end{figure}

\label{sec:human-interpretability}
In this section, we discuss 
\emph{i)} if the discovered tasks should be visually interpretable to humans,
\emph{ii)} if one should expect them to contain human-labelled tasks and if they do in practice, and
\emph{iii)} an example of how the discovery of human-labelled tasks can be promoted.

\myparagraph{Should the discovered high-\AStext tasks be human-interpretable?}
In this work, by ``human-interpretable tasks'', we generally refer to those classifications that humans can visually understand or learn sufficiently conveniently. Human-labelled tasks are examples of such tasks.
While we found in \secref{sec:as-real-random-study} that such interpretable tasks have a high \AStext, the opposite is not true, i.e., not \textit{all} high-\AStext tasks should be visually interpretable by a human -- as  they reflect the inductive biases of the particular learning algorithm used, which are not necessarily aligned with human perception.
The following ``green pixel task'' is an example of such a high-\AStext task that is learnable and generalizable in a statistical learning sense but not convenient for humans to learn visually.

\textbf{The \emph{``green pixel task''}.} Consider the following simple task: the label for $x$ is 1 if the pixel $p$ at a fixed position has the green channel intensity above a certain threshold and 0 otherwise.
The threshold is chosen to split images evenly into two classes.
This simple task has the \AStext of approximately $0.98$, and a network trained to solve it has a high test accuracy of $0.96$.
Moreover, the network indeed seems to make the predictions based on this rule rather than other cues: the accuracy remains almost the same when we set all pixels but $p$ to random noise and, on the flip side, drops to the chance level of 0.5 when we randomly sample $p$ and keep the rest of the image untouched.
This suggests that the network captured the underlying rule for the task and generalizes well in the statistical learning sense.
However, it would be hard for a human to infer this pattern by only looking at examples of images from both classes, and consequently, it would appear like an uninterpretable/random task to human eyes.

It is sensible that many of the discovered tasks belong to such a category. 
This indicates that finding more human-interpretable tasks would essentially require \emph{bringing in additional constraints and biases} that the current neural network architectures do not include.
We provide an example of such a promotion using the SSL regulated \taskdiscovery results below.

\myparagraph{Do discovered tasks contain human-labelled tasks?}
% To answer this question, for each available human-labelled task $\treal  \in \Treal$, we find the most similar discovered task $\hat{\t}_d \in \Tset_{\theta_e^*}$ (see \appendixref{app:real-recall} for more details).
We observe that the similarity between the discovered and most human-labelled tasks is relatively low (see \figref{fig:app-real-recall}{-left} and \appendixref{app:real-recall} for more details).
% We use encoders with different embedding dimensionalities trained the same way as in \secref{sec:td-cifar}.
% \figref{fig:app-real-recall}{-left} shows that the similarity remains relatively low for most human-labelled tasks for these dimensionalities.
As mentioned above, human-labelled tasks make up only a small subset of all tasks with a high \AStext.
The \taskdiscovery framework aims to find different (i.e., dissimilar) tasks from this set and not necessarily \textit{all} of them.
In other words, there are many tasks with a high \AStext other than human-labelled ones, which the proposed discovery framework successfully finds.
% Note that it is especially important to discover tasks different from human-labelled ones in the context of adversarial splits introduced in the next section to construct balanced train-test splits.

As mentioned above, introducing additional inductive biases would be necessary to promote finding more human-labelled tasks.
We demonstrate this by using the tasks discovered with the SimCLR pre-trained encoder (see \secref{sec:td-study-qualitative}).
\figref{fig:app-real-recall}{-right} shows that the recall of human-labelled tasks among the discovered ones increases notably due to the inductive biases that SimCLR data augmentations bring in.

\section{Adversarial Dataset Splits}
\label{sec:adversarial-split}
\begin{figure}
    \centering
    \includegraphics[width=\textwidth]{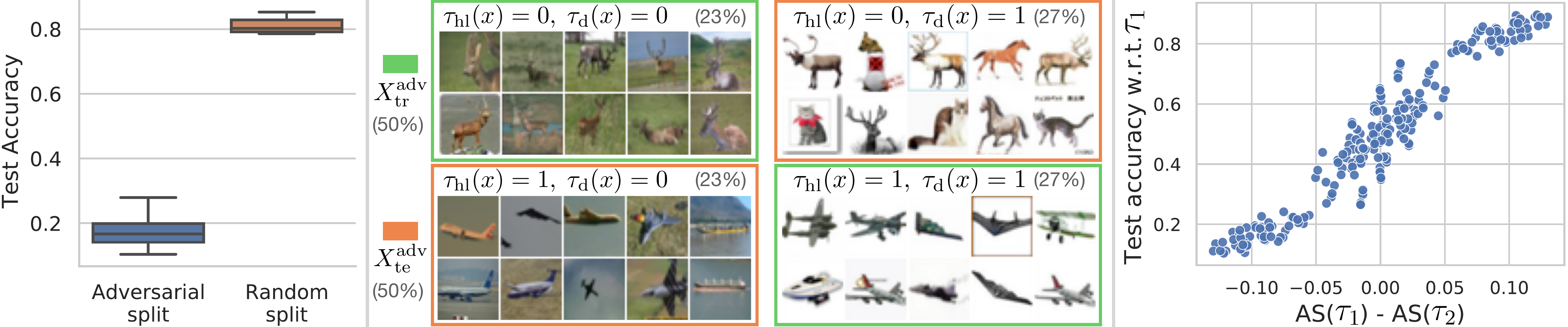}
    \caption{
    \figtitle{Adversarial splits on CIFAR-10.}
    \figpos{Left:} Test accuracy for human-labelled tasks $\treal$ on two types of splits: random and adversarial based on a $\tdiscovered$ (see the center image and \eqref{eq:adv-split}). 
    The boxplot distribution is over four $\treal$ tasks for the random split and 24  $(\treal, \tdiscovered)$ pairs for the adversarial one. Training with the adversarial split results in a {significant} test performance drop.
    \figpos{Center:} Example of an adversarial split,
    constructed such that images from \green{train} and \orange{test} sets have the opposite correlation between $\treal$ and $\tdiscovered$ (the percentage of images in each group is shown in brackets).
    The model seems to learn to make its predictions based on $\tdiscovered$ instead of $\treal$ (hence, the significant test accuracy drop).
    \figpos{Right:} Each dot corresponds to a pair of tasks $(\t_1, \t_2)$, where
    the \textit{$y$-axis} is the test accuracy after training on the corresponding adversarial split to predict $\t_1$, and the \textit{$x$-axis} is the difference in the \AStext of these two tasks.
    The plot suggests that the network \textit{favors learning the task with a higher AS}: the higher the \AStext of $\t_1$ is, the higher the accuracy w.r.t. $\t_1$ is.
    Thus, for an adversarial split to be successful, the AS of the target task $\t_1$ should be lower than the \AStext of $\t_2$.
    \figcutspace
    }
    \label{fig:adversarial-split}
\end{figure}

In this section, we demonstrate how the discovered tasks can be used to reveal biases and failure modes of a given network architecture (and training pipeline).
In particular, we introduce an \textit{adversarial split} -- which is a train-test dataset split such that the test accuracy of the yielded network significantly drops, compared to a random train-test split. This is because the \AStext can reveal what tasks a network \emph{favors} to learn, and thus, the discovered tasks can be used to construct such adversarial splits.

\subsection{Creating Adversarial Train-Test Dataset Partitions Using Discovered High-\AStext Tasks}
For a given human-labelled task $\treal$, let us consider the standard procedure of training a network on $\D(\Xt, \treal)$ and testing it on $\D(\Xte, \treal)$.
The network usually achieves a high test accuracy when the dataset is split into train and test sets randomly.
Using discovered tasks, we show how to construct \textit{adversarial} splits on which the test accuracy drops significantly.
To do that, we take a discovered task $\tdiscovered$ with high a \AStext and construct the split, s.t. $\treal$ and $\tdiscovered$ have the same labels on the training set $\Xtadv$ and the opposite ones on the test set $\Xteadv$ (see \figref{fig:adversarial-split}{-center}):
\begin{equation}
    \label{eq:adv-split}
    \Xtadv = \{x \,|\, \treal(x) = \tdiscovered(x), \, x\in X\}, \;\; \Xteadv = \{x \,|\, \treal(x) \neq \tdiscovered(x), \, x\in X\}.
\end{equation}
\figref{fig:adversarial-split}{-left} shows that for various pairs of $(\treal, \tdiscovered)$, the test accuracy on $\D(\Xteadv, \treal)$ drops significantly after training on $\D(\Xtadv, \treal)$.
This suggests the network chooses to learn the cue in $\tdiscovered$, rather than $\treal$, as it predicts $1 - \treal$ on $\Xteadv$, which coincides with $\tdiscovered$.
We note that we keep the class balance and sizes of the random and adversarial splits approximately the same (see \figref{fig:adversarial-split}{-center}).
The train and test sets sizes are approximately the same because we use discovered tasks that are different from the target human-labelled ones as discussed in \secref{sec:human-interpretability} and ensured by the dissimilarity term in the \taskdiscovery objective in \eqref{eq:task-discovery-obj}.

The discovered task $\tdiscovered$, in this case, can be seen as a spurious feature \cite{sagawa2020investigation,khani2021removing} and the adversarial split creates a spurious correlation between $\treal$ and $\tdiscovered$ on $\Xtadv$, that fools the network.
Similar behaviour was observed before on datasets where spurious correlations were curated manually~\cite{sagawa2019distributionally,liang2022metashift,wiles2021fine}.
In contrast, the described approach using the discovered tasks allows us to find such spurious features, to which networks are vulnerable, \textit{automatically}.
It can potentially find spurious correlations on datasets where none was known to exist or find new ones on existing benchmarks, as shown below.

\myparagraph{Neural networks favor learning the task with a higher AS.}
The empirical observation above demonstrates that, when a network is trained on a dataset where $\treal$ and $\tdiscovered$ coincide, it predicts the latter.
While theoretical analysis is needed to understand the cause of this phenomenon, here, we provide an empirical clue towards its understanding.

We consider a set of 10 discovered tasks
% (5 from each of the discovery runs with $d\in \{8, 32\}$) 
and 4 human-labelled tasks.
For all pairs of tasks $(\t_1, \t_2)$ from this set, we construct the adversarial split according to \eqref{eq:adv-split}, train a network on $\D(\Xtadv, \t_1)$ and test it on $\D(\Xteadv, \t_1)$.
\figref{fig:adversarial-split}{-right} shows the test accuracy against the difference in the agreement scores of these two tasks $\AS(\t_1) - \AS(\t_2)$.
We find that the test accuracy w.r.t. $\t_1$ correlates well with the difference in the agreement: when $\AS(\t_1)$ is sufficiently larger than $\AS(\t_2)$, the network makes predictions according to $\t_1$, and vice-versa.
When the agreement scores are similar, the network makes test predictions according to neither of them (see \appendixref{app:adversarial-splits} for further discussion).
This observation suggests that an adversarial split is successful, i.e., causes the network to fail, if the AS of the target task $\t_1$ is lower than that of the task $\t_2$ used to create the split.

\subsection{Adversarial Splits for Multi-Class Classification.}
\label{sec:adv-split-multi-class}
In this section we show one way to extend the proposed adversarial splits to a multi-way classification target task.
Let us consider a multi-class task $\t: X \to C$,  where $C=\{1, \dots, K\}$, for which we want to construct an adversarial split.
We cannot create a ``complete'' correlation between $\t$ and a binary discovered task similar to \eqref{eq:adv-split} as they have different numbers of classes.
On the other hand, using a $K$-way discovered task will result in having too few images in the corresponding train set (where $\tdiscovered(x) = \t(x)$).
Instead, we suggest creating only a partial correlation between $\t$ and a binary discovered task $\tdiscovered$.
To do that, we split the set of all classes $C$ into two disjoint sets $C_1, C_2$ and create the an adversarial split in the following way:
\begin{equation}
    \label{eq:k-way-adv}
    \Xtadv = \{x \,|\, \mathbb{I}[\t(x) \in C_{\boldsymbol{1}}] = \tdiscovered(x), \, x\in X\}, \;\; \Xteadv = \{x \,|\, \mathbb{I}[\t(x) \in C_{\boldsymbol{2}}] = \tdiscovered(x), \, x\in X\},
\end{equation}
% where $\mathbb{I}[\cdot]$ is the indicator function.
In this case, $\tdiscovered(x)$ does not equate $\t(x)$ on $\Xtadv$, but it is predictive of whether $\t(x)\in C_1$ or $\t(x)\in C_2$, creating a partial correlation.
Intuitively, if $\tdiscovered$ provides a good learning signal, the network can learn to distinguish between $C_1$ and $C_2$ first and then perform classification inside each set of classes.
In this the case, the test set will fool the network as it has the correlation opposite to the train set.

\tabref{tab:adv-split} shows results of attacking the original 10-way semantic classification task from CIFAR-10.
We can see that the proposed adversarial splits based on the discovered high-\AStext tasks generalize well to the multi-class classification case, i.e., the accuracy drops noticeably compared to random splits.
Note that the network, in this case, is trained on only half the CIFAR-10 original training set. Hence, the accuracy is 0.8 for a random split as opposed to around 0.9 when trained on the full dataset.

\subsection{Adversarial Splits for ImageNet and CelebA via Random-Network Task.}
\begin{table}
    \caption{
    % \todo{caption should be on top}. 
    The test set accuracy of networks trained on the original multi-class tasks for CIFAR-10 and ImageNet, and on \texttt{hair\_blond} attribute for CelebA.
    To create adversarial splits, we used discovered tasks for the CIFAR-10 dataset and tasks corresponding to randomly-initialized networks for ImageNet and CelebA.
    The class balance and the sizes of train and test splits are kept the same for random and adversarial splits.
    The standard deviation is over training runs, class splits for CIFAR-10 and ImageNet (see \secref{sec:adv-split-multi-class}) and different random splits.
    The test performance drops significantly for all three datasets when trained and tested on adversarial splits.
    \figcutspace
    }
    \label{tab:adv-split}
    \centering
    \begin{tabular}{@{}lccc@{}}
    \toprule
    Split       & CIFAR-10      & CelebA           & ImageNet, top-1 \\ \midrule
    Random      & $0.78\pm0.04$ & $0.94\pm0.00$    & $0.59\pm0.01$      \\
    Adversarial & $0.41\pm0.10$ & $0.42\pm0.02$    & $0.29\pm0.00$      \\ \bottomrule
    \end{tabular}
\end{table}
We demonstrate the effectiveness of the proposed approach and create adversarial splits for ImageNet~\cite{deng2009imagenet} and CelebA~\cite{liu2015deep}, a popular dataset for research on spurious correlations.
To get a high-AS task, we utilize the phenomenon observed in \secref{sec:td-cifar} and use a randomly initialized network as $\tdiscovered$.
We use ResNet-50 for ImageNet and ResNet-18 for CelebA.
We do not use augmentations for training.
Both adversarial and random splits partition the original dataset equally into train and test, and, therefore, the networks are trained using half of the original training set.
\tabref{tab:adv-split} shows that adversarial splits created with the task corresponding to a randomly initialized network lead to a significant accuracy drop compared to random train-test splits.
See \appendixref{app:adv-split-celeba-imgnet-viz} for visualization of these adversarial splits.

\vspace{-6pt}
\section{Conclusion}
\vspace{-6pt}

In this work, we introduce \emph{\taskdiscovery}, a framework that finds tasks on which a given network architecture generalizes well.
It uses the Agreement Score (\AStext) as the measure of generalization and optimizes it over the space of tasks to find generalizable ones.
We show the effectiveness of this approach and demonstrate multiple examples of such generalizable tasks.
We find that these tasks are not limited to human-labelled ones and can be based on other patterns in the data.
This framework provides an empirical tool to analyze neural networks through the lens of the tasks they generalize on and can potentially help us better understand deep learning.
Below we outline a few research directions that can benefit from the proposed task discovery framework.

\textbf{Understanding neural networks' inductive biases.}
Discovered tasks can be seen as a reflection of the inductive biases of a learning algorithm (network architectures, optimization with SGD, etc.), i.e., a set of preferences that allow them to generalize on these tasks.
Therefore, having access to these biases in the form of concrete tasks could help us understand them better and guide the development of deep learning frameworks.

\textbf{Understanding data.}
As discussed and shown, task discovery depends on, not only the learning model, but also the data in hand.
Through the analysis of discovered tasks, one can, for example, find patterns in data that interact with a model's inductive biases and affect its performance, thus use the insights to guide dataset collection.

\textbf{Generalization under distribution shifts.}
The \AStext turned out to be predictive of the cues/tasks a network favors in learning. The consequent adversarial splits (\secref{sec:adversarial-split}) provide a tool for studying the biases and generalization of neural networks under distribution shifts.
They can be constructed automatically for datasets where no adverse distribution shifts are known and help us to build more broad-scale benchmarks and more robust models.

\vspace{-6pt}
\section{Limitations}
\vspace{-6pt}
The proposed instantiation of a more general \taskdiscovery framework has several limitations, which we outline below, along with potential approaches to address them.

\textbf{Completeness:} the set of discovered tasks does not necessarily include \emph{all} of the tasks with a high \AStext.
Novel optimization methods that better traverse different optima of the optimization objective \eqref{eq:task-discovery-obj}, e.g., \cite{parker2020ridge, lorraine2021lyapunov}, and further scaling are needed to address this aspect.
Also, while the proposed encoder-based parametrization yields an efficient \taskdiscovery method, it imposes some constraints on the discovered tasks, as there might not exist an encoder such that the corresponding set of tasks $\Tset_{\theta_e}$ contains \emph{all and only} high-\AStext ones.

\textbf{Scalability:} 
the \taskdiscovery framework relies on an expensive meta-optimization, which limits its applicability to large-scale datasets.
This problem can be potentially addressed in future works with recent advances in efficient meta-optimization methods \cite{lorraine2020optimizing, raghu2021meta} as well as employing effective approximations of the current processes.

\textbf{Interpretability:}
As we discussed and demonstrated, analysis of discovered high-\AStext tasks can shed more light on how the neural networks work.
However, it is the expected behavior that not all of these tasks may be easily interpretable or ``useful'' in the conventional sense (e.g. as an unsupervised pre-training task). This is more of a consequence of the learning model under discovery, rather than the task discovery framework itself. Having discovered tasks that exhibit such properties requires additional inductive biases to be incorporated in the learning model. This was discussed in \secref{sec:human-interpretability}.

\vspace{-4pt}
\paragraph{Acknowledgment:}
We thank Alexander Sax, Pang Wei Koh, Roman Bachmann, David Mizrahi and Oğuzhan Fatih Kar for the feedback on earlier drafts of this manuscript. We also thank Onur Beker and Polina Kirichenko for the helpful discussions.

\bibliographystyle{plain} % We choose the "plain" reference style
\bibliography{reference} % Entries are in the refs.bib file

\appendix

\newpage

\section{Task Discovery on Tiny ImageNet}
\label{app:tiny-imganet-discovery}
\begin{figure}
    \centering
    \includegraphics[width=0.95\textwidth]{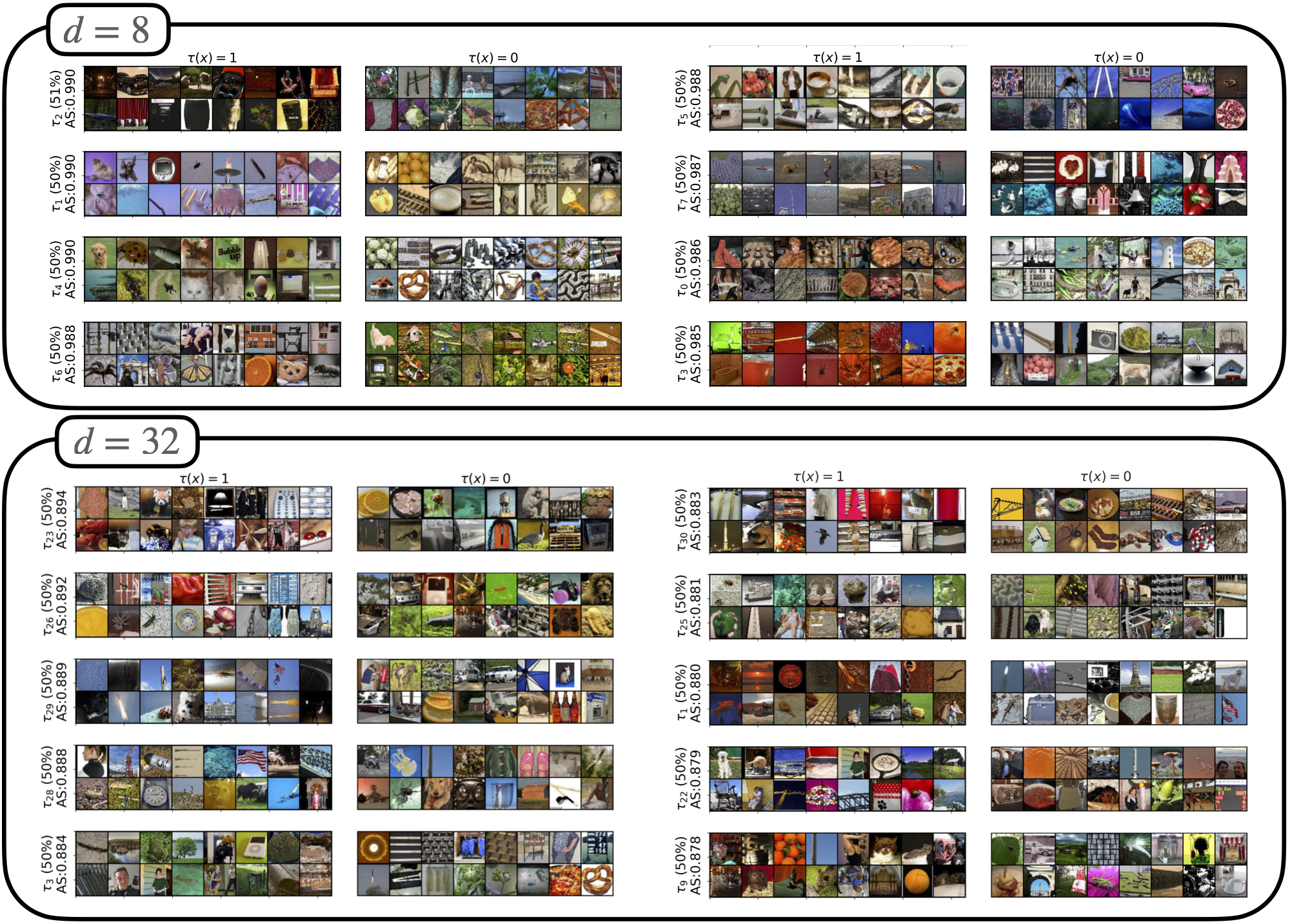}
    \caption{
    \figtitle{Tasks Visualization for TinyImageNet} for $d=\{8,32\}$. The top 10 tasks for each $d$, as measured by \AStext is shown (if $d<10$, then $d$ tasks are selected). Each column and row shows selected images from a task, for class 1 ($\tau(x)=1$) and class 0 ($\tau(x)=0$). In the $y$-axis, we show the fraction of images in class 1 in brackets and the \AStext for that task. The images for each class have been selected to be the most discriminative, as measured by the network's predicted probabilities.
    }
    \label{fig:tiny-imagenet-tasks}
\end{figure}
We extend our experimental setting by scaling it along the dataset size axis and run the proposed task discovery framework on the TinyImageNet dataset \cite{chrabaszcz2017downsampled}.
We use the same framework with the shared embedding space formulation with $d\in \{8, 32\}$ and ResNet18 with global pooling after the last convolutional layer.
\figref{fig:tiny-imagenet-tasks} shows examples of discovered tasks.
All discovered tasks have \AStext above 0.8, whereas human-labelled binary tasks (constructed using original classes in the same way as for CIFAR-10 as described in Sec.~3.2) have \AStext below 0.65. We assume that this may be due to the fact that TinyImageNet contains semantically close classes. Binary task can assign different labels for these classes, which forces the agreement to be lower because it is more difficult for model learn on how to discriminate these classes.

% which suggests the considered architecture lacks inductive biases needed for training semantic human-labelled tasks on this dataset.

\section{Vizualization of the Tasks Discovered on CIFAR-10}
\label{app:vizualization}
\begin{figure}
    \centering
    \includegraphics[width=\textwidth]{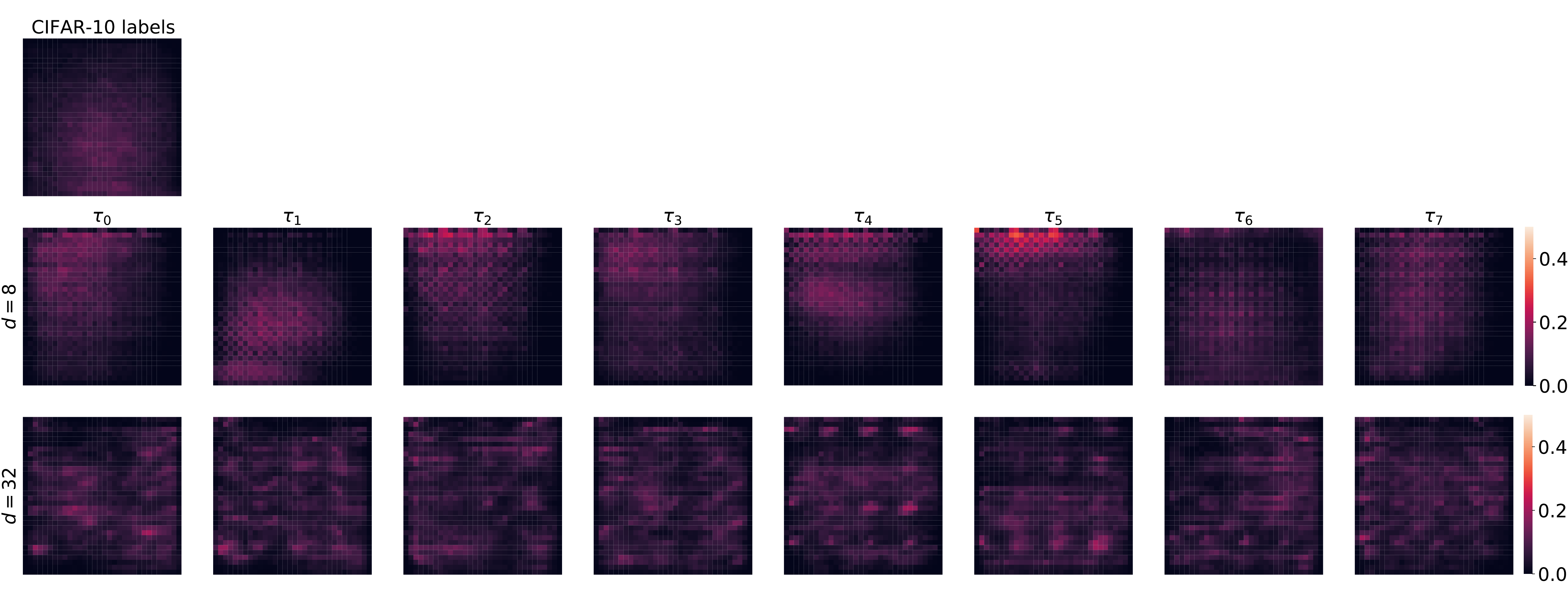}
    \caption{
    \figtitle{Sufficient input subsets for tasks networks for $d=\{8,32\}$.} Each plot shows the heatmap of the frequency that each pixel was selected to be in the sufficient subset. The task network with $d=8$ seems to attend to more macro features in the image, while with $d=32$, it seems to use cues for generalization that are scatter over the image. 
    }
    \label{fig:sis}
\end{figure}

We provide a complete set of visuals from task networks with different embedding space dimensions ($d = \{8, 32, 128\}$). The top 10 tasks for each $d$, as measured by AS, are shown in~\figref{fig:task_viz_samples}{} and the full set of tasks can be found in the \texttt{supmat\_taskviz} folder in the supplementary material. Visuals for the regulated version of task discovery, with an encoder pre-trained with SimCLR~\cite{chen2020simple}, is also shown in~\figref{fig:task_viz_samples}. The task network used for visualizations in Fig. 4 in the main paper has $d = 8$. As in the main paper, the images for each class have been selected to be the most discriminative, as measured by the network’s predicted probabilities. As $d$ increases, the less interpretable the tasks seem to be. With lower values of $d$, the tasks seem to be based on color patterns. Refer to~\secref{sec:human-interpretability} for a discussion as to whether it is reasonable to expect these tasks to be interpretable.

We also attempt to analyze how the task network makes its decisions. To do this, we use a post hoc interpretability method, Sufficient Input Subsets (SIS) \cite{carter2019made}. SIS aims to find the minimal subset of the input, i.e. pixels, whose values suffice for the model to make the same decision as on the original input. Running SIS gives us a ranking of each pixel in the image. We assign the value 1 to the top 5\% of the pixels and 0 to all the others. Thus, each image results in a binary mask. Similar to \cite{carter2021overinterpretation}, we aggregate these binary masks for each task for tasks networks with $d = \{8, 32\}$ to get the heatmaps shown in~\figref{fig:sis}. This allows us to see if different tasks are biased towards different areas of the image. Note that this only tells us where the network may be looking at in the image, but not how it is using that information. For comparison, we also show the SIS map from training on the CIFAR-10 original classification task and a random-labelled task. As expected, the heatmap for CIFAR-10 is roughly centred, as objects in the images are usually in the centre. Furthermore, increasing $d$ results in heatmaps that are more scattered. This seems to suggest that task networks discovered with a smaller embedding size use cues in the image that are more localized.

The labels of the images in Fig. 4 (last column) of the main paper are as follows: $\tau_1=[0, 1, 0, 1, 1, 0]$, $\tau_2=[1, 1, 1, 0, 0, 0]$, $\tau_3=[0, 0, 1, 1, 1, 0]$, $\tau_1^{reg.}=[1, 0, 0, 1, 0, 1]$, $\tau_2^{reg.}=[0, 1, 0, 0, 1, 1]$, $\tau_3^{reg.}=[0, 1, 1, 1, 0, 0]$. 
% \newgeometry{left=1.5cm,right=1.5cm,bottom=0.01cm,top=0.01cm}

% \restoregeometry
\begin{figure}
    \centering
    \includegraphics[width=\textwidth]{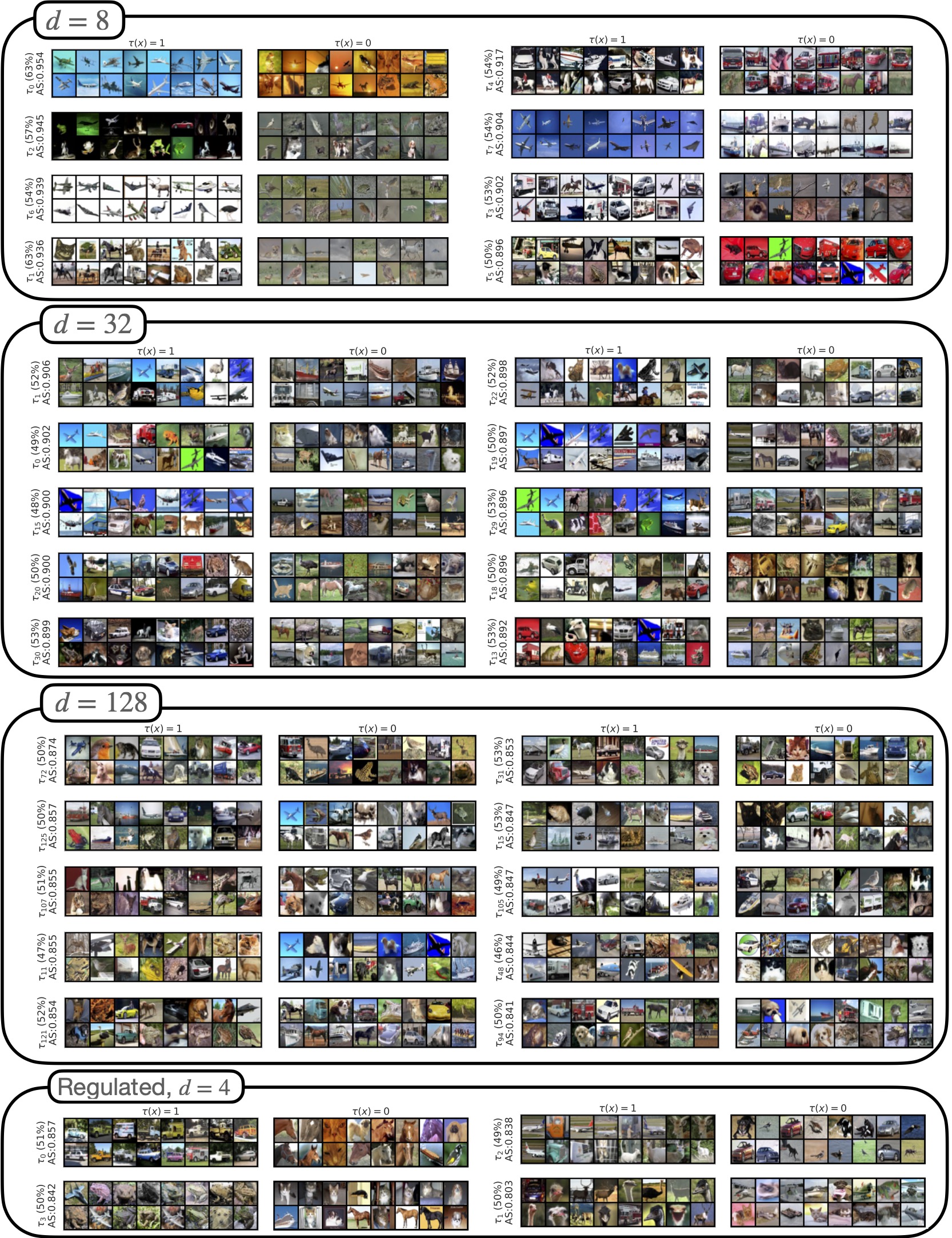}
    \caption{
    \figtitle{Tasks visualizations} for $d=\{8,32,128\}$ and the regulated version of task discovery. The top 10 tasks for each $d$, as measured by \AStext is shown (if $d<10$, then $d$ tasks are selected). Each column and row shows selected images from a task, for class 1 ($\tau(x)=1$) and class 0 ($\tau(x)=0$). In the $y$-axis, we show the fraction of images in class 1 in brackets and the \AStext for that task. The images for each class have been selected to be the most discriminative, as measured by the network's predicted probabilities.
    For unregulated tasks, the higher $d$ is, the less (immediately) interpretable the tasks seem to be. For low $d$, the tasks seem to be based on color patterns. On the other hand, the regulated tasks seem to be based on semantic information.
    Refer to~\secref{sec:human-interpretability} for a discussion as to whether it is reasonable to expect these tasks to be interpretable.
    }
    \label{fig:task_viz_samples}
\end{figure}

\section{Discovering Multi-Class Classification Tasks}
\label{app:td-k-way-cifar}
\begin{figure}
    \centering
    \includegraphics[width=0.95\textwidth]{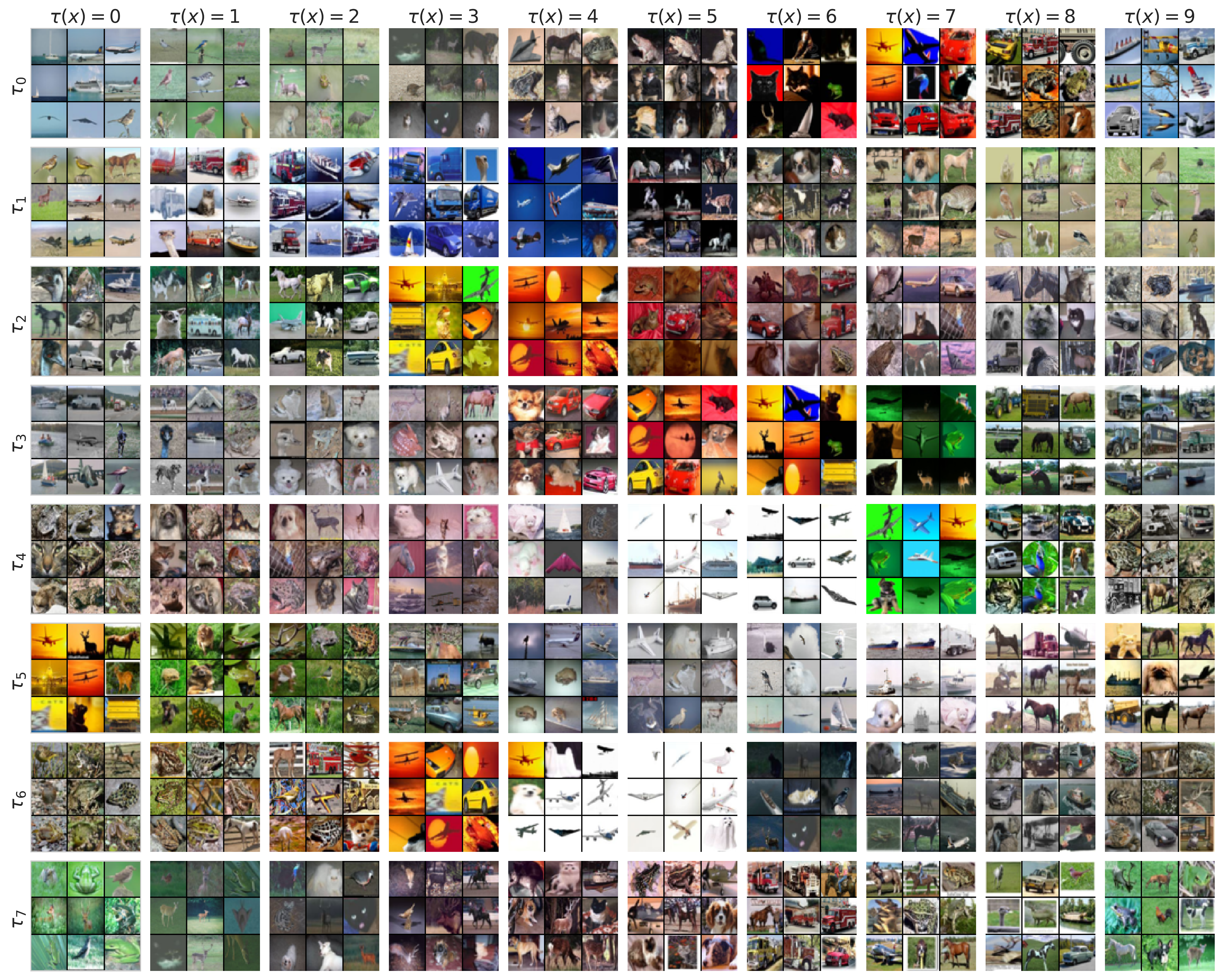}
    \caption{
    \figtitle{Visualisation of 10-way Task Discovery on CIFAR-10.} Each column and row shows selected images from a task, for class $i=\{0,\ldots 9\}$ ($\tau(x)=i$). The images for each class have been selected to be the most discriminative, as measured by the network's predicted probabilities.
    }
    \label{fig:td-k-way-cifar}
\end{figure}

In this section, we show how the \taskdiscovery framework with a shared embedding space proposed in 
\secref{sec:td-encoder-formulation}
can be extended to the case of multi-class classification tasks.
That is, instead of discovering high-AS binary tasks, we are interested in finding a $K$-way classification task $\tau:\Xspace \to \{0, \dots, K\}$.
Note that the formulation of the agreement score (Eq.1 in the main paper) and its differentiable approximation 
(see \secref{sec:as-meta-opt} and \secref{app:meta-opt}) 
% (see Sec.~4.1 and \secref{app:meta-opt}) 
remain almost the same.
To model the set of tasks $T$, we use the same shared encoder formulation as in \secref{sec:td-encoder-formulation}, but with the linear layer predicting $K$ logits, i.e., $\theta_l \in \R^{d\times K}$ and optimize the same loss with the uniformity regularizer.
In order to sample tasks where classes are balanced (a similar number of images in each class), we construct a linear layer $\theta_l = [\theta_l^1, \dots, \theta_l^K]$, s.t. $\angle(\theta_l^i, \theta_l^{i+1}) = 2\pi/K$, by randomly sampling $\theta_l^1$ and a direction along which we rotate it.
Note, that if the uniformity constraint is satisfied, then the fraction of images corresponding to each class will be $\angle(\theta_l^i, \theta_l^{i+1})/2\pi = 1/K$, i.e., classes will be balanced.

\section{Adversarial Splits Vizualization for CelebA and ImageNet}
\label{app:adv-split-celeba-imgnet-viz}
\figref{fig:adv-split-celeba-imgnet-viz}{} shows examples of images from the adversarial splits for CelebA and ImageNet datasets.
\begin{figure}
    \centering
    \includegraphics[width=\textwidth]{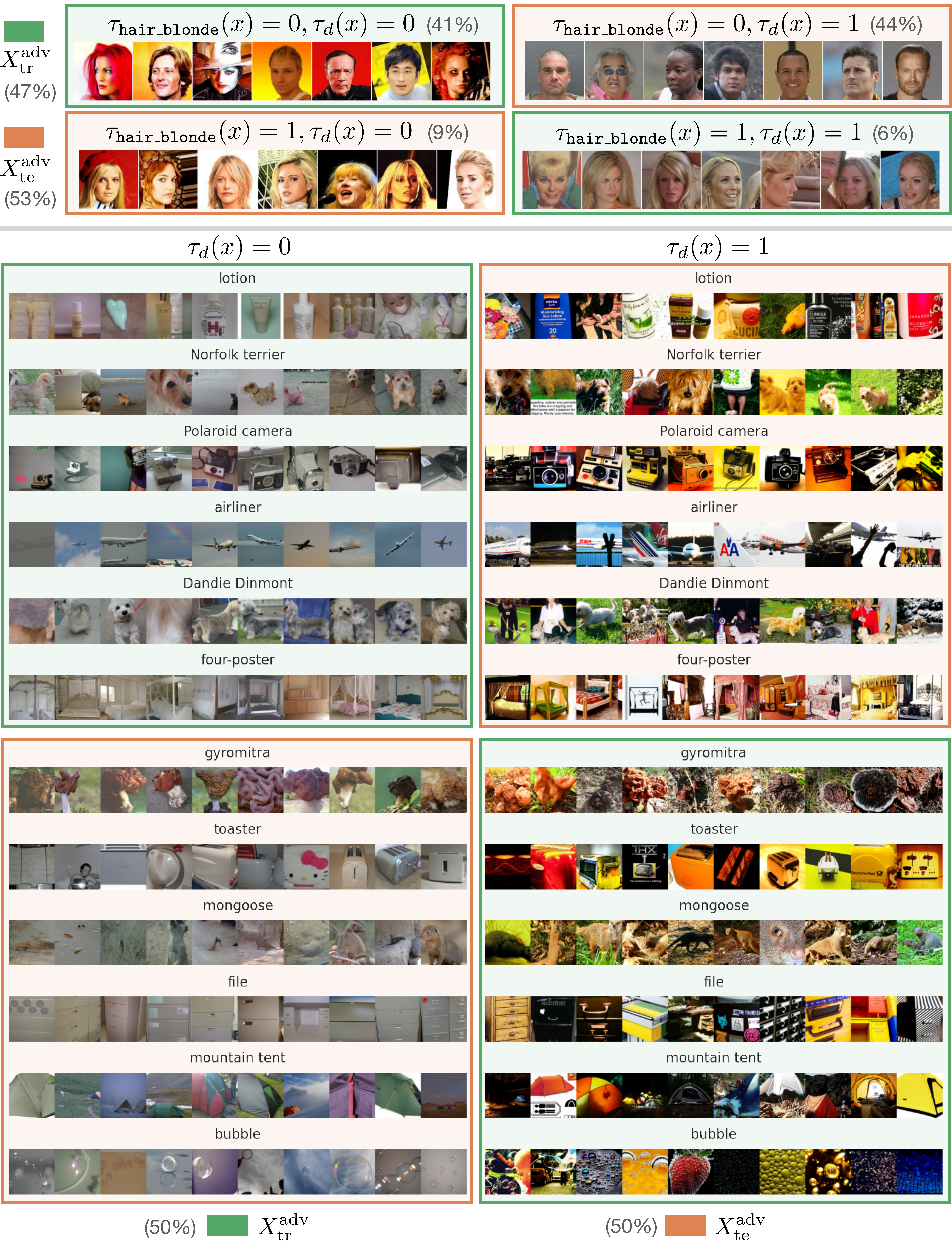}
    \caption{Visualization of adversarial \green{train}-\orange{test} splits for CelebA (top) and ImageNet (bottom) datasets.
    To construct these splits we use high-\AStext task corresponding to a randomly initialized network (see \secref{sec:td-cifar} and \secref{app:random-net}).
    }
    \label{fig:adv-split-celeba-imgnet-viz}
\end{figure}

\section{$\lambda$ Hyperparameter: Trade-off between the Agreement Score and Similarity}
\begin{figure}
    \centering
    \includegraphics[width=0.5\textwidth]{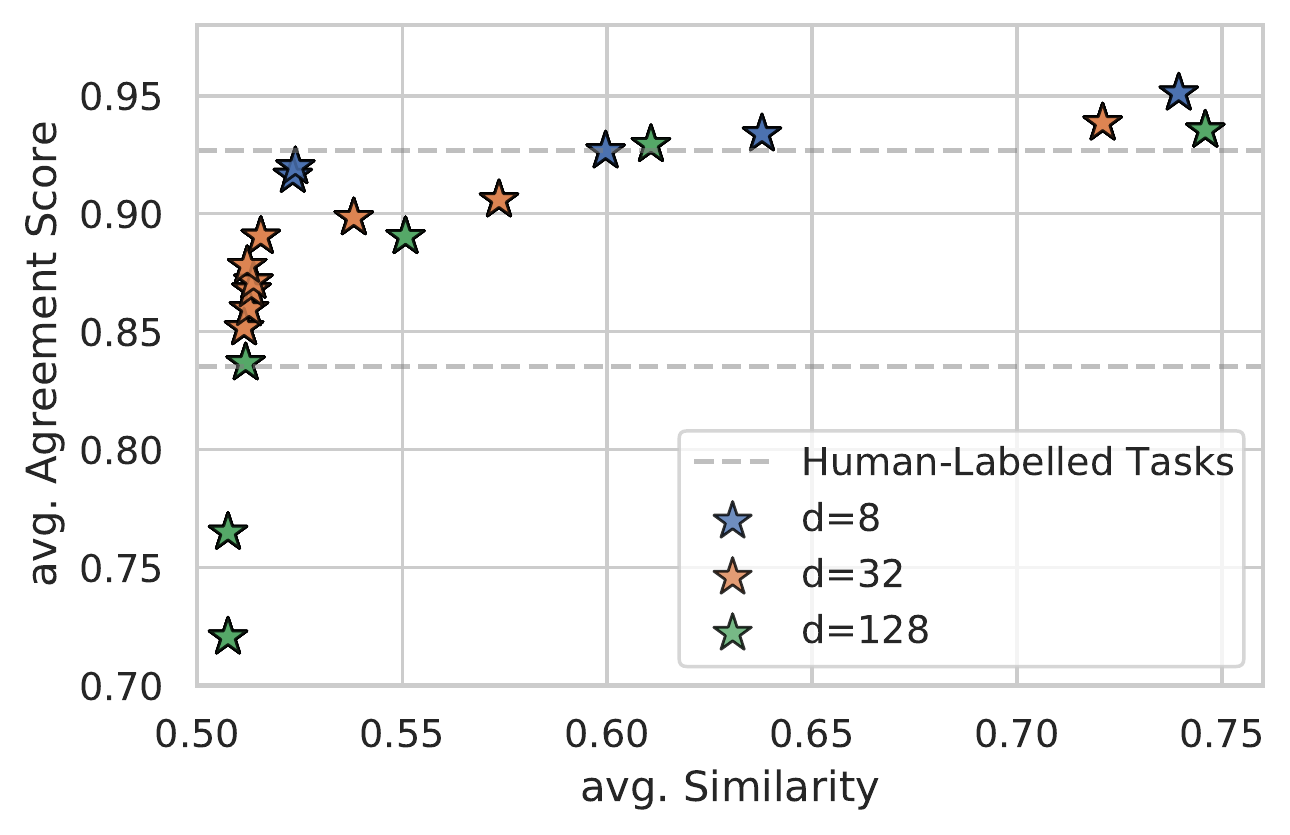}
    \caption{
    \figtitle{The Trade-off between the Agreement Score and Similarity induced by the $\lambda$ hyperparameter.}
    Each dot corresponds to a set of task $T$ discovered with the corresponding embedding space dimensionality and parameter $\lambda$ (dots with higher similarity have higher value of $\lambda$).
    The $x$ coordinate is computed by measuring the maximum similarity for each task in $T$ to the tasks from the same set and averaging these values, i.e.,
    $1/|T| \sum_{t\in T} \max_{t' \in T} \similarity(t, t')$.
    The $y$ coordinate is the average \AStext over tasks from $T$.
    We can see that the method is not very sensitive to $\lambda$ (for the considered dimensionalities) and most runs have high average \AStext even when similarity is low.
    }
    \label{fig:as-similarity-curve}
\end{figure}
\label{app:lambda-hyperparam-as-similarity-curve}
The \taskdiscovery optimization objective (Eq.~3) has two terms: 1) average agreement score over the set of discovered tasks $T$ and 2) similarity loss, measuring how similar tasks are to each other.
The hyperparameter $\lambda$ controls the balance between these two losses.
If $\lambda = 0$, a trivial solution would be to have all the tasks the same with the highest possible AS.
On the other hand, when $\lambda \to \infty$, tasks in $T$ will be dissimilar to each other with lower average AS (since the number of high-AS tasks is limited).
One can tune the hyperparameter to fit a specific goal based on whether many different tasks are needed or a few but with the highest agreement score.

\figref{fig:as-similarity-curve} shows how such a trade-off looks in practice for the shared embedding formulation with different dimensionalities of the embedding space.
We can see that, in fact, the proposed method for task discovery is not very sensitive to the choice of $\lambda$ as the \AStext mostly stays high even when the similarity is close to its minimum value.

\section{Can Networks Trained on Adversarial Splits Generalize?}
\label{app:adversarial-splits}
\begin{figure}
    \centering
    \includegraphics[width=0.6\textwidth]{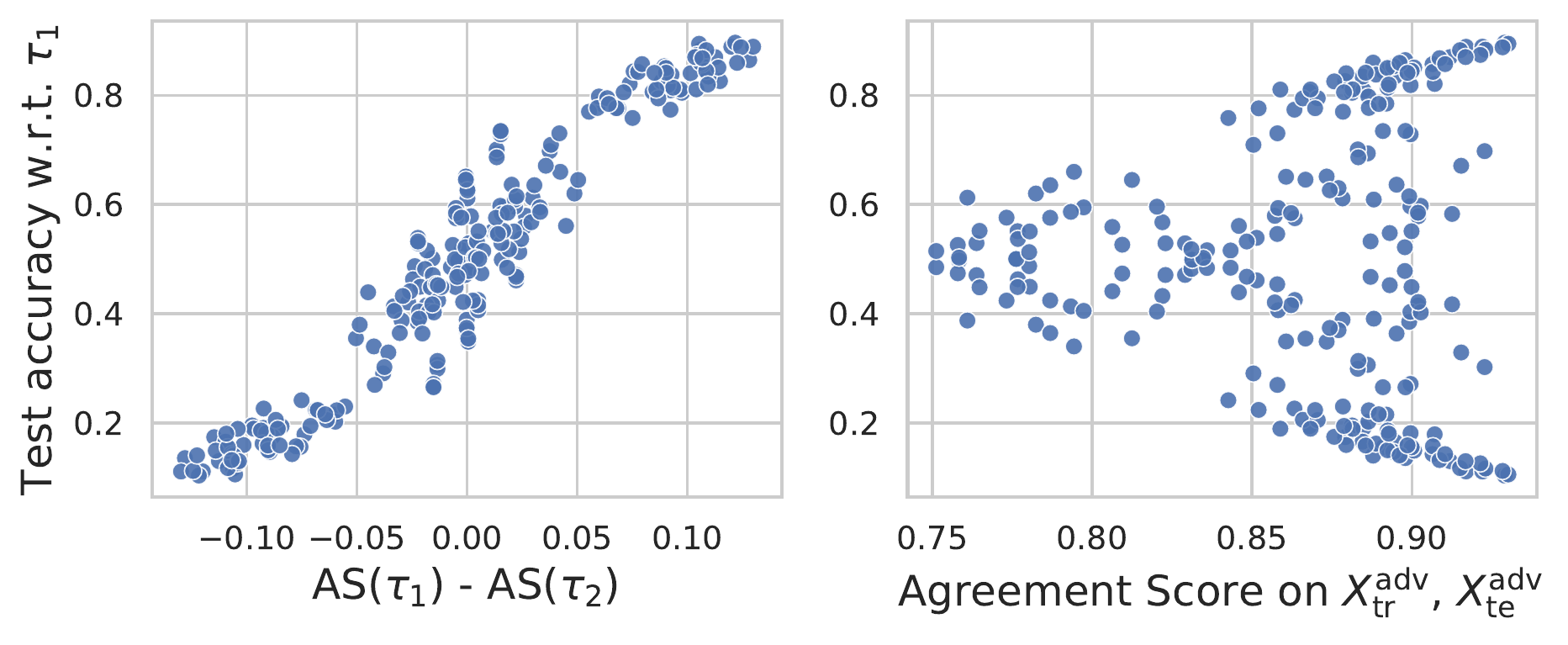}
    \caption{
    % \figtitle{The dependence of the adversarial split test accuracy on the difference in the agreement scores and the \AStext}
    \figpos{Left:} Each dot corresponds to a pair of tasks $(\t_1, \t_2)$.
    The \textit{$y$-axis} is the test accuracy on $\D(\Xteadv, \t_1)$ after training on the corresponding adversarial split $\D(\Xtadv, \t_1)$.
    The \textit{$x$-axis} is the difference in the agreement score of two tasks averaged over random train-test splits.
    This is the same plot as in \figref{fig:adversarial-split}{-right}.
    \figpos{Center:} The same data but with \textit{$x$-axis} being the \AStext measured on the adversarial split, i.e., $\AS(\t_1; \Xtadv, \Xteadv)\, (=\AS(\t_2; \Xtadv, \Xteadv))$.
    The symmetry on the plots is due to the interchangeable role of $\t_1$ and $\t_2$ as they are binary tasks, which only changes the value of the test accuracy ($\mathrm{acc} \to \mathrm{1-acc}$).
    We see that even when the test accuracy is close to 0.5 w.r.t. to both $\t_1$ and $\t_2$, the \AStext stays high in most cases, suggesting that networks generalize well when trained on $\D(\Xtadv, \t_1)$ but w.r.t. other labels on test images from $\Xteadv$ (see Proposition~\ref{prop:as-bound}).
    % \figpos{Right:}
    % Adversarial splits for CIFAR-10 10-way classification task.
    % Please see \secref{app:cifar-10-adv} for details.
    % The accuracy is w.r.t. the original 10-way classification task of CIFAR-10 dataset.
    % Standard deviation is over different class splits $C_1, C_2$ and discovered tasks $\tdiscovered$.
    % \textbf{We can see that the proposed adversarial train-test splits successfully generalizes to a $K$-way classification task.}
    }
    \label{fig:app-adv-splits}
\end{figure}

In \secref{sec:adversarial-split} of the main paper, we introduce adversarial train-test splits that ``fool'' a network, significantly reducing the test accuracy compared to randomly sampled splits.
However, it remains unclear whether networks can potentially exhibit generalization when trained on these adversarial splits w.r.t. other test labels or if the behaviour is similar to training on a random-labelled task, i.e., different networks converge to different solutions.
In order to rule out the latter hypothesis, we measure the \AStext of a task on the adversarial train-test split and show that it stays high, i.e., the task stays generalizable (in the view of Proposition~\ref{prop:as-bound}). 

Recall that for a pair of tasks $(\t_1, \t_2)$, the adversarial split is defined as follows:
\begin{equation}
    \label{eq:app-adv-split}
    \Xtadv = \{x \,|\, \t_1(x) = \t_2(x), \, x\in X\}, \;\; \Xteadv = \{x \,|\, \t_1(x) \neq \t_2(x), \, x\in X\},
\end{equation}
the only difference with the definition from \secref{sec:adversarial-split}, \eqref{eq:adv-split}, is that here we consider a general case with $\t_1,\,\t_2$ being any two tasks not only a human-labelled or discovered.
In \secref{sec:adversarial-split}, we show that the test accuracy on $\D(\Xtadv, \t_1)$ seems to correlate well with the difference in agreement scores between two tasks (see \figref{fig:app-adv-splits}).
In particular, when the agreement scores are close to each other, i.e., the difference is close to zero, the test accuracy w.r.t. both tasks is close to $0.5$.
We further investigate if networks, in this case, can generalize well w.r.t. some other labelling of the test set $\Xteadv$, i.e., whether the agreement score on this split is high.

\figref{fig:app-adv-splits}{-right} shows that even when the test accuracy is close to $0.5$, the \AStext measured on the adversarial split, i.e., $\AS(\t_1; \Xtadv, \Xteadv)$, stays high in most cases.
This observation means that networks trained on $\D(\Xtadv, \t_1)$ do generalize well in the sense of Proposition~\ref{prop:as-bound}. That is, they converge to a consistent solution, which, however, differs from both $\t_1$ and $\t_2$.
% according to other test labels with the test accuracy of at least $\AS(\t_1; \Xtadv, \Xteadv)$ (see Proposition~\ref{prop:as-bound}).

\section{The Connection between the Agreement Score and Test Accuracy}
\label{app:as-accuracy-connection}
In this section, we establish the connection between the test accuracy and the agreement score for the fixed train-test split $\Xt, \Xte$, mentioned in \secref{sec:as}.
We show that the agreement score provides a lower bound on the ``best-case'' test accuracy but cannot predict test accuracy in general.

For a task $\t$ and a train-test split $\Xt, \Xte$, we consider the test accuracy to be the accuracy on the test set $\D(\Xte, \t)$ averaged over multiple models trained on the same training dataset $\D(\Xt, \t)$:
\begin{equation}
    \label{eq:acc}
    \ACC(\t; \Xt, \Xte) = \E_{w\sim \A(\D(\Xt, \t))} \E_{x\sim \Xte} [f(x; w) = \t(x)],
\end{equation}
where $\E_{x\sim \Xte}$ stands for averaging over the test set.
We will further omit $\Xt, \Xte$ and simply use $\ACC(\t)$ since the split remains fixed.
Recall that for a learning algorithm $\A$, we define the agreement score (AS) as follows (see \secref{sec:as}):
\begin{equation}
    % \label{eq:as}
    % \AS(\t, X) = \E_{X_t, X_v} \E_{w_1, w_2 \sim \A(X_t, \t)} \frac{1}{N_v} \sum_{x \in X_v} [f(x; w_1) = f(x; w_2)]
    \AS(\t; \Xt, \Xte) = \E_{w_1, w_2 \sim \A(\D(\Xt, \t))} \E_{x\sim \Xte} [f(x; w_1) = f(x; w_2)].
\end{equation}

Intuitively, when the agreement score is high, models trained on $\D(\Xt, \t)$ make similar predictions, and, therefore, there should be some labelling of test data $\Xte$, such that the accuracy w.r.t. these labels is high.
We formalize this intuition in the following proposition.
\begin{prop}
\label{prop:as-bound}
For a given train-test split $\Xt, \Xte$, a task defined on the training set $\t: \Xt \to \{0, 1\}$ and a learning algorithm $\A$, the following inequalities hold for any $h:\Xte \to \Yspace$:
\begin{equation}
    \label{eq:as-bounds}
    \ACC(\hat{\t}) \geq \AS(\t) \geq 2\ACC(\Tilde{\t}) - 1,
\end{equation}
where
\[
    \hat{\t} = \left\{\begin{array}{lr}
        \t(x),& x\in \Xt \\
        \argmax_c \E_{w\sim \A(\D(\Xt, \t))} [f(x; w) = c],& x \in \Xte
        \end{array}\right.,
\]
and
\[
    \Tilde{\t} = \left\{\begin{array}{lr}
        \t(x),& x\in \Xt \\
        h(x),& x \in \Xte
        \end{array}\right..
\]
\end{prop}
\begin{proof}
1) Let us, first, show that the first inequality holds true.
\begin{align*}
    \ACC(\hat{\t}) 
    &= \E_{w\sim \A(\D(\Xt, \t))} \E_{x\sim \Xte} [f(x; w) = \hat{\t}(x)]\\
    &= \E_{w\sim \A(\D(\Xt, \t))} \E_{x\sim \Xte} [f(x; w) = \argmax_c \E_{\Bar{w}\sim \A(\D(\Xt, \t))} [f(x; \Bar{w}) = c]]\\
    &=  \E_{x\sim \Xte}  \E_{w\sim \A(\D(\Xt, \t))} [f(x; w) = \argmax_c \E_{\Bar{w}\sim \A(\D(\Xt, \t))} [f(x; \Bar{w}) = c]]\\
    &=  \E_{x\sim \Xte}  \max_c \E_{w\sim \A(\D(\Xt, \t))} [f(x; w) = c]
\end{align*}
Due to the maximum, for any $h:\:\Xte\to \{0, 1\}$, the following holds for any $x \in \Xte$:
\begin{equation*}
\max_c \E_{w\sim \A(\D(\Xt, \t))} [f(x; w) = c] \geq \E_{w\sim \A(\D(\Xt, \t))} [f(x; w) = h(x)],
\end{equation*}
and, therefore:
\begin{equation}
    \label{eq:max-ineq}
    \E_{x\sim \Xte}  \max_c \E_{w\sim \A(\D(\Xt, \t))} [f(x; w) = c] \geq \E_{x\sim \Xte} \E_{w\sim \A(\D(\Xt, \t))}  [f(x; w) = h(x)].
\end{equation}
In particular, \eqref{eq:max-ineq} holds for $f(\cdot\,; w_2)$ for any $w_2\sim\A(\D(\Xt, \t))$, and, therefore:

\begin{gather*}
    \ACC(\hat{\t}) \geq \E_{x\sim \Xte} \E_{w\sim \A(\D(\Xt, \t))}  [f(x; w) = f(x; w_2)]\\
    \E_{w_2\sim \A(\D(\Xt, \t))} \ACC(\hat{\t}) \geq \E_{w_2\sim \A(\D(\Xt, \t))} \E_{x\sim \Xte} \E_{w\sim \A(\D(\Xt, \t))}  [f(x; w) = f(x; w_2)]\\
    \ACC(\hat{\t}) \geq \E_{w, w_2 \sim \A(\D(\Xt, \t))} \E_{x\sim \Xte} [f(x; w) = f(x; w_2)]\\
    \ACC(\hat{\t}) \geq \AS(\t)\;[=\AS(\hat{\t})]
\end{gather*}
% \begin{align*}
%     \ACC(\hat{\t})
%     &= \E_{x\sim \Xte}  \max_c \E_{w\sim \A(\D(\Xt, \t))} [f(x; w) = c]\\
%     &\geq \E_{x\sim \Xte} \E_{w\sim \A(\D(\Xt, \t))}  [f(x; w) = f(x; w_2)]\\
%     &[\ACC~\text{does not depend on}~w_2]\\
%     &= \E_{w_2\sim \A(\D(\Xt, \t))} \E_{x\sim \Xte} \E_{w\sim \A(\D(\Xt, \t))}  [f(x; w) = f(x; w_2)]\\
%     &=\E_{w, w_2 \sim \A(\D(\Xt, \t))} \E_{x\sim \Xte} [f(x; w) = f(x; w_2)]\\
%     &=\AS(\t)\;[=\AS(\hat{\t})]
% \end{align*}

2) Note that  $\forall a,b,c\in \{0, 1\}$ holds $[a=b] \geq [a=c]+[b=c] -1$.
Then:
\begin{align*}
    \AS(\t) 
    &= \E_{w_1, w_2 \sim \A(\D(\Xt, \t))} \E_{x\sim \Xte} [f(x; w_1) = f(x; w_2)]\\
    &\geq \E_{w_1, w_2 \sim \A(\D(\Xt, \t))} \E_{x\sim \Xte} [f(x; w_1) = \Tilde{\t}] +  [\Tilde{\t} = f(x; w_2)] - 1\\
    &\geq \E_{x\sim \Xte} \left( \E_{w_1\sim \A(\D(\Xt, \t))} [f(x; w_1) = \Tilde{\t}] +  \E_{w_2\sim \A(\D(\Xt, \t))} [\Tilde{\t} = f(x; w_2)]\right) - 1\\
    &=\ACC(\Tilde{\t}) + \ACC(\Tilde{\t}) - 1\\
    &= 2 \ACC(\Tilde{\t}) - 1
\end{align*}
\end{proof}

Proposition~\ref{prop:as-bound} establishes a connection between the \AStext and the test accuracy measured on the same test set $\Xte$.
It can be easily seen from the first inequality in \eqref{eq:as-bounds} that the \AStext provides only a \textit{necessary} condition for high test accuracy, and, therefore, \AStext cannot be predictive of the test accuracy in general \cite{jiang2021assessing} as was also noted in \cite{kirsch2022note}.
Indeed, test accuracy will be high if test labels are in accordance with $\hat{\t}$ and low if, for example, they correspond to $1-\hat{\t}$. 

As can also be seen from the second inequality in \eqref{eq:as-bounds}, one can directly optimize the $\ACC(\Tilde{\t})$ w.r.t. $\Tilde{\t}$ (both the training and test labels) to find a generalizable task.
However, compared to optimizing the \AStext, this would additionally require modelling test labels of the task on $\Xte$, whereas \AStext only requires training labels on $\Xt$ and, hence, has fewer parameters to optimize.

\section{On Random Networks as high-\AStext Tasks}
\label{app:random-net}
\begin{figure}
    \centering
    \includegraphics[width=\textwidth]{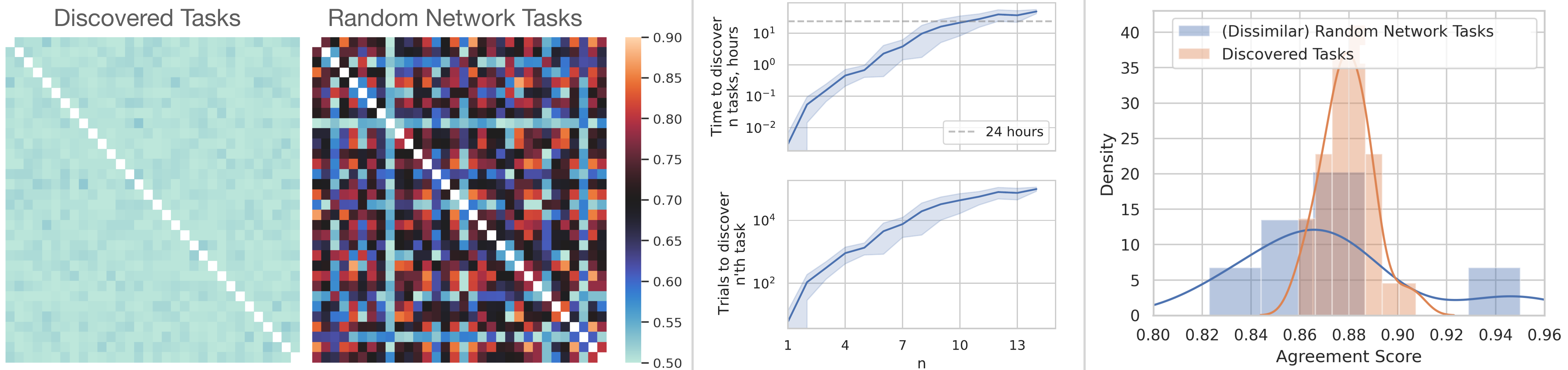}
    \caption{
    \figpos{Left:} The pairwise similarity matrix between 32 tasks from the set of discovered and random network tasks.
    \textbf{We see that random network tasks are much more similiar to each other.}
    \figpos{Center:} Sampling random networks to discover dissimilar tasks.
    $x$-axis: the discovered task number, $y$-axis: hours to discover $n$ tasks (top) and the number of generated random networks to obtain $n$th dissimilar task (bottom).
    We use the similarity threshold $0.55$ for this experiment.
    The standard deviation is over 5 runs with different initial seeds.
    \figpos{Right:} The distribution of the \AStext for the set of the found (dissimilar) random network tasks and discovered tasks from \secref{sec:td-cifar}.
    We see that while random networks give rise to high-\AStext tasks, we cannot control for their similarity and, hence, it is and inefficient \taskdiscovery framework.
    Also, the \AStext drops a bit compared to tasks discovered with the proposed \taskdiscovery framework.
    }
    \label{fig:app-random-net}
\end{figure}

In \secref{sec:td-cifar} of the main paper, we found that the \AStext of a randomly initialized network is high.
This section provides more details about this experiment and additional discussion of the results.

\subsection{Constructing the Task Corresponding to a Randomly Initialized Network}
\label{app:random-net-as}
When initializing the network randomly and taking the labels after the softmax layer, the corresponding task is usually (as found empirically) unbalanced with all labels being either 0 or 1, which trivially is a high-\AStext task but not of interest.
We, therefore, apply a slightly different procedure to construct a task corresponding to a randomly initialized network.
First, we collect the logits of the randomly initialized network for each image and sort them.
Then, we split images evenly into two classes according to this ordering, which results in the corresponding labelling from a \textit{random network task}.
\figref{fig:app-random-net}{-right} shows the \AStext for a set of such tasks (we discuss how we construct this set in the next \secref{app:random-net-td}).
We can see that tasks corresponding to randomly initialized networks indeed have high \AStext. 
We use the same ResNet-18 architecture with the same initialization scheme as for measuring the \AStext (see \secref{app:experimental-setup}).

\subsection{Drawing Random Networks as a Naive Task Discovery Method}
\label{app:random-net-td}
Based on the observation made in \secref{app:random-net-as}, one could suggest drawing random networks as a naive method for \taskdiscovery.
The major problem, however, is that two random network tasks are likely to have high similarity in terms of their labels (see \figref{fig:app-random-net}{-left}), whereas in \taskdiscovery we want to find a set of more diverse tasks.
A simple modification to do so is to sequentially generate random network tasks as described in \secref{app:random-net-as}, and keep the task if its maximum similarity to previously saved tasks is below a certain threshold.
We perform this procedure with the threshold of $0.55$. 
\figref{fig:app-random-net}{-center} shows that this naive discovery method is inefficient.
For comparison, our proposed \taskdiscovery from \secref{sec:td-cifar} takes approximately 24 hours to train.
Also, we can see that the \AStext of (dissimilar) random network tasks seems to deteriorate compared to the set of tasks discovered via meta-optimization.

% \section{On Human-Interpretability of the Discovered Tasks and Discovering Human-Labelled Tasks}

\section{Experimental Setup For Task Discovery}
\label{app:experimental-setup}
In this section, we provide more technical details on the proposed \taskdiscovery framework, including a discussion on a differentiable version of the \AStext in \secref{app:proxy-as}, meta-optimization setting in \secref{app:meta-opt} and the description of architectures in \secref{app:arch}.
% \aanote{emphasise on the number of inner-loop steps used.}

% \subsection{Limiting the Number of Steps in the Inner-Loop Optimization}
\subsection{Modifying the \AStext for Meta-Optimization}
\label{app:proxy-as}

\begin{figure}
    \centering
    \includegraphics[width=\textwidth]{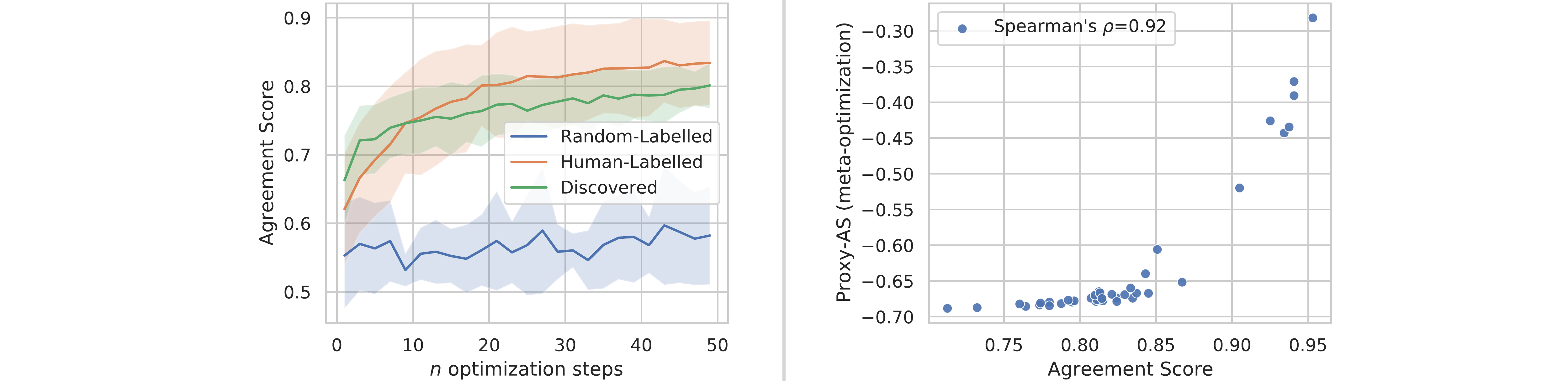}
    \caption{
    \figtitle{Agreement score approximation for meta-optimization.}
    \figpos{Left:}
    The agreement score at every iteration of the inner-loop optimization of two networks for \blue{random-labelled}, \orange{human-labelled} and \green{discovered} tasks.
    Standard deviation for each group is over 5 tasks and 3 seeds (accounts for initialization and data-order).
    We can see that the differentiation between high-AS (\orange{human-labelled} and \green{discovered}) tasks and \blue{random-labelled} occurs early in training and the \AStext after only 50 steps can be reliably used for \taskdiscovery.
    \figpos{Right:}
    The dependence between the proxy-\AStext (50 steps/soft labels/(negative)cross-entropy agreement/SGD) and the original Agreement Score (100 epochs/hard labels/0-1 agreement/Adam).
    We see that the two correlate well with Spearman's rank correlation of 0.92 making it a good optimization proxy.
    }
    \label{fig:proxy-as}
\end{figure}

In order to be able to optimize the \AStext defined in \secref{sec:as} we need to make it differentiable and feasible, i.e., reduce its computational and memory costs.
We achieve this by applying the following changes to the original \AStext:
\begin{itemize}
    \item We use (negative) cross-entropy loss instead of the 0-1 loss to measure the agreement between two networks after training.
    \item To train two networks in the inner-loop, we use ``soft'' labels of the task-network $t_\theta$, i.e., probabilities, instead of ``hard'' $\{0, 1\}$ labels. Since the task networks influence the inner-loop through the training labels, we can backpropagate through these ``soft'' labels to the parameters $\theta$.
    \item We use SGD instead of Adam optimizer in the inner-loop as it does not require storing additional momentum parameters and, hence, has lower memory cost and is more stable in the meta-optimization setting.
    \item We limit the number of inner-loop optimization steps to 50.
\end{itemize}
We refer to the corresponding \AStext with these modifications as \textit{proxy-\AStext}.
In this section, we validate whether the proxy-\AStext is a good objective to optimize the original \AStext.

First, we verify that 50 steps are enough to provide the differentiation between human-labelled and random-labelled tasks.
\figref{fig:proxy-as}{-left} shows that the differentiation between high-\AStext human-labelled and discovered tasks and random-labelled tasks occurs early in training, and 50 steps are enough to capture the difference.
The results of \taskdiscovery also support this as the found tasks have high \AStext when two networks are trained till convergence for 100 epochs (see \figref{fig:pull} of the main paper).
Further, \figref{fig:proxy-as}{-right} shows that the proxy-\AStext (with all modifications applied) correlates well with the original \AStext and, therefore, makes for a good optimization proxy.

\subsection{Task Discovery Meta-Optimization Details}
\begin{algorithm}
\caption{Pseudo-Code for Task Discovery with Shared Embedding Space}\label{alg:xas}
    \begin{algorithmic}
        \State initialize the task-encoder weights $\theta_e^0$
        \State initialize orthogonal task-specific linear heads $\{\theta_l^i\}_{i=1}^{d}$
        \For{$t$ in $0,\dots, T-1$} \Comment{outer-loop}
            \State $\theta_l \sim \{\theta_l^i\}_{i=1}^{d}$ \Comment{sample a task-specific head}
            \State $w_1^0, w_2^0 \gets p(w)$ \Comment{initialize two networks weights randomly}
            \State $L_{\mathrm{AS}}(\theta_e^t), L_{\mathrm{unif}}(\theta_e^t) \gets 0, 0$
            \For{$k$ in $0, \dots, K-1$} \Comment{inner-loop (assume less than 1 epoch)}
                \State sample a training batch $[x_1, \dots, x_M]$
                \State $L_{\mathrm{AS}}(\theta_e^t) \gets L_{\mathrm{AS}}(\theta_e^t) + \frac{1}{M} \sum_{i=1}^M \xeloss(f(x_i; w_1^k), f(x_i; w_2^k))$ \Comment{$x_i$s are novel for $f$s}
                \State $h_i \gets e(x_i;\,\theta_e^t), \, i=1,\dots,M$ 
                \State $L_{\mathrm{unif}}(\theta_e^t) \gets L_{\mathrm{unif}}(\theta_e^t) + \mathcal{L}_{\mathrm{unif}}(h)$
                \State $t_i \gets \sigma(\theta_l^\top h_i), \, \, i=1,\dots,M$ \Comment{get task ``soft'' labels}
                \State $w_1^{k+1} \gets w_1^{k} - \alpha \cdot \nabla_w \frac{1}{M} \sum_{i=1}^M \xeloss(f(x_i; w_1^k), t_i)$
                \State $w_2^{k+1} \gets w_2^{k} - \alpha \cdot \nabla_w \frac{1}{M} \sum_{i=1}^M \xeloss(f(x_i; w_2^k), t_i)$
            \EndFor
            \State $\theta_e^{t+1} \gets \theta_e^t - \eta \cdot \nabla_{\theta_e} (L_{\mathrm{AS}}(\theta_e^t) + L_{\mathrm{unif}}(\theta_e^t))$ \Comment{use accumulated \AStext and uniformity losses}
        \EndFor
        \State \textbf{return} $\theta_e^T, \{\theta_l^i\}_{i=1}^{d}$
    \end{algorithmic}
\end{algorithm}
\label{app:meta-opt}
For all \taskdiscovery experiments, we use the following setup.
We use the same architecture as the backbone for both the networks $f(\cdot\,;w)$ in the inner-loop and the task encoder $e(\cdot\,;\theta_e)$, which has an additional linear projection layer to $\R^d$.
To model tasks, we use $d$ predefined orthogonal hyper-planes $\{\theta_l^i\}_{i=1}^d$, which remain fixed throughout the training.
Then at each meta-optimization step, we sample a task corresponding to one of these $d$ hyper-planes $\theta_l$ and update the encoder by using the gradient of the \AStext w.r.t. encoder parameters: $\nabla_{\theta_e}\AS(t_{(\theta_e, \theta_l)})$.
We found it to be enough to optimize the \AStext w.r.t. these tasks corresponding to the $d$ predefined hyper-planes to train the encoder for which \textit{any} randomly sampled hyper-plane $\Tilde{\theta}_l$ gives rise to a task with high \AStext.

% For dissimilarity optimization we calculate uniformity loss on each step \afnote{mention temperature}.
We accumulate the uniformity loss (\secref{sec:td-encoder-formulation}, \eqref{eq:uniformity}) for embeddings of all images passed through the encoder during the inner-loop and backpropagate through it once at each meta-optimization step.
We set the weight $\lambda$ of the uniformity loss for the main setting with $d=32$ to $0.66$ for ResNet-18, $5$ for MLP and $100$ for ViT.
These weights were chosen by trying multiple values and ensuring that similarities between 32 discovered tasks for each architecture are approximately the same between $0.5$ and $0.55$.
We use the SGD optimizer with \texttt{lr=1e-2} and batch size $512$ for the inner-loop and Adam with \texttt{lr=1e-3} as the meta-optimizer for the outer-loop. We use the \texttt{higher} library \cite{grefenstette2019generalized} for meta-optimizaion.

To further speed up the \taskdiscovery framework, we early-stop the inner-loop optimization when the \AStext between two networks is above $0.6$ for more than three consecutive steps.
Also, since we limit the number of steps to 50, every training batch in the inner-loop is novel for two networks and can be seen as validation data.
We utilize this and accumulate the proxy-\AStext between two networks on these training batches (before applying the gradient update on them) and use it as the final proxy-\AStext for backpropagation and updating the encoder.
See Algorithm~\ref{alg:xas}.

% After inner-optimization we optimization we calculate cross-entropy between predictions of learners. 
% In inner loop we also use cross-entropy for learner's  optimization. For meta-optimization we use SGD with \texttt{lr=1e-3}. For inner optimization we use SGD with \texttt{lr=1e-3}. 

% \aanote{describe \textit{why} we use different similartiy coefs, basically, to promote similar similarity}
% For ViT architecture we set similiraty coef $\lambda=100$. We train it for \afnote{find number epochs or iterations}.
% For ResNet18 architecture we set  similarity coef $\lambda=0.66$. We train it for \afnote{find number epochs or iterations}. \afnote{Describe XAS}
% For MLP architecture we use 87 inner-steps which is equal to one epoch. We set similarity coef $\lambda=5$. We train it for \afnote{find number epochs or iterations}. 

\subsection{Architectural Details}
\label{app:arch}
In all cases, we initialize models' weights $w_1,w_2$ with Kaiming uniform initialization \cite{he2015delving}, which is the default initialization in PyTorch \cite{NEURIPS2019_9015}.

\begin{table}[]
\begin{tabular}{|l|l|c|c|}
\hline
Module name       & Module                                                         & in      & out  \\ \hline
Layer 1    & Linear($3072, 1024$), ReLU, BN($1024$) & $3072$ & $1024$ \\ \hline
Layer 2    & Linear($1024, 512$), ReLU, BN($512$)      & $1024$   & $512$  \\ \hline
Layer 3    & Linear($512, 256$), ReLU, BN($256$)       & $512$     & $256$  \\ \hline
Layer 4    & Linear($256, 64$), ReLU                              & $256$     & $64$   \\ \hline
Classifier & Linear($64, 2$)    & $64$      & $2$    \\ \hline
\end{tabular}
\caption{MLP Architecture}
\label{app:mlp-arch}
\end{table}

For the MLP, we use the 4-layer perceptron with batch normalization (see \tabref{app:mlp-arch}).

We use standard ResNet18 \cite{he2016deep} architecture adapted for CIFAR10 dataset: the first 7x7 conv with stride=2 is replaced by 3x3 conv with stride=1. 

We use ViT architecture in the following way. We split an image onto 4x4 patches. We remove dropout and replace all LayerNorm layers with BatchNorm to achieve faster convergence. We reduce the embedding dimensionality and the MLP hidden layers' dimensionality from 512 to 256 and utilize only six transformer blocks to fit the model into memory for task discovery. After the transformer, we feed class embedding to a linear classifier.

\section{Do Discovered Tasks Contain Human-Labelled Tasks?}
\label{app:real-recall}
Here, we describe in more detail how we answer this question.
In order to answer it, for each human-labelled task $\treal  \in \Treal$, we find the most similar discovered task $\hat{\t}_d \in \Tset_{\theta_e^*}$ for the encoder $e(\cdot\,;\theta_e^*)$ from \secref{sec:td-cifar}.
We, however, cannot simply list all the discovered tasks from this set and compare them against all human-labelled ones since this set is virtually infinitely large (any hyperplane is supposed to give rise to a high-\AStext task).
Therefore, for each task $\treal$, we fit a linear classifier $\theta_l^r$ using embeddings as inputs and $\treal$ as the target.
We use the corresponding task $\hat{\t}_d(x) = \theta_l^{r\top} e(x; \theta_e^*)$ as (an approximation of) the most similar task from the set of discovered tasks $\Tset_{\theta_e^*}$.
\figref{fig:app-real-recall}{-left} shows that the similarity increases with scaling the embedding size but remains relatively low for most human-labelled tasks for the scale at which we ran our experiments.

\end{document}